\documentclass{article}

\pdfoutput=1

\usepackage[nonatbib,preprint]{neurips_2020}

\newif\ifpreprint
\makeatletter\def\ifpreprint{\if@preprint}\makeatother

\usepackage[largesc]{newtxtext}
\usepackage[T1]{fontenc} % best for Western European languages
\usepackage{amsmath,mathtools}
\usepackage{amsthm,thmtools,thm-restate}
\usepackage{amssymb} % loads amsfonts
\usepackage{nicefrac}
\usepackage[smallerops,varg]{newtxmath}
\usepackage{mathrsfs}
\usepackage{dsfont}
\usepackage{bm} % load after all math to give access to bold math

\usepackage{xifthen}
\usepackage{etoolbox}

\usepackage{microtype}
\usepackage[inline,shortlabels]{enumitem}
\usepackage[dvipsnames]{xcolor}
\usepackage{setspace}
\usepackage{booktabs}
\usepackage{algpseudocode,algorithm}
\usepackage[numbers,compress]{natbib}

\usepackage{varioref} % Provides vref
\PassOptionsToPackage{hyphens}{url} % loaded by hyperref
\usepackage[pdfusetitle,bookmarks]{hyperref}
\usepackage{bookmark}
\usepackage[doipre=doi:,mailtopre=]{uri} % load after hyperref; URLs for DOI, arXiv, etc
\usepackage[capitalise,nameinlink]{cleveref} % load after varioref, hyperref

\hypersetup{ %
  pdfborder=0 0 0,
  pdfpagemode=UseNone,
  colorlinks=true,
  pdfstartview=FitH,
  bookmarksnumbered,
  bookmarksopen,
  bookmarksopenlevel=2
}

\usepackage[
disable,
,obeyFinal
,textsize=small]{todonotes}
\makeatletter
\renewcommand{\todo}[2][]{\@todo[noline,size=\scriptsize,#1]{#2}}
\makeatother
\newcommand{\todoc}[2][]{\todo[color=blue!20!white,#1]{Csaba: #2}}
\newcommand{\todor}[2][]{\todo[color=orange!20!white,#1]{Roshan: #2}}

\providecommand\phantomsection[1]{}

\mathtoolsset{mathic}

\hypersetup{%
 linkcolor=Purple,
 citecolor=RoyalPurple,
 filecolor=MidnightBlue,
 urlcolor=MidnightBlue,
}

\newtagform{colored}{\color{WildStrawberry}(}{)}
\usetagform{colored}

\newcommand{\ifequals}[2]{\ifthenelse{\equal{#1}{#2}}}
\providecommand{\eprint}[2][]{%
  \ifequals{#1}{arXiv}{\arxiv{#2}}
  {\ifthenelse{\isempty{#1}}{#2}{#1:#2}}%
}

\Crefname{claim}{Claim}{Claims}

\crefname{footnote}{footnote}{footnotes}
\crefname{page}{page}{pages}

\crefname{equation}{}{}

\newtagform{uncolored}{(}{)}
\makeatletter
\creflabelformat{equation}{\usetagform{uncolored}#2\tagform@{#1}#3}
\makeatother

\def\eqref{\cref}

\renewcommand{\epsilon}{\varepsilon}

\renewcommand{\vec}[1]{\bm{#1}}
\newcommand{\wildcard}{\mathinner{\,{\cdot}\,}}
\newcommand{\defeq}{\coloneq}
\newcommand{\eqdef}{\eqcolon}
\newcommand{\inv}[1]{#1^{-1}}
\newcommand{\Real}{\mathds{R}}

\renewcommand{\Pr}{\mathds{P}}

\newcommand{\argmin}{\operatorname*{arg\,min}}

\newcommand\given[1][\delimsize]{%
  \providecommand{\delimsize}{}
  \nonscript\:#1\vert\allowbreak\nonscript\:\mathopen{}
}

\DeclarePairedDelimiter{\abs}\lvert\rvert
\DeclarePairedDelimiter{\paren}{\lparen}{\rparen}
\DeclarePairedDelimiter{\brck}{\lbrack}{\rbrack}
\DeclarePairedDelimiterX{\set}[1]\lbrace\rbrace{#1}

\DeclarePairedDelimiterX{\innerp}[2]\langle\rangle{#1,#2}
\DeclarePairedDelimiterXPP{\Prob}[1]{\Pr}{\lparen}{\rparen}{}{#1}
\DeclarePairedDelimiterXPP{\PrSet}[1]{\Pr}{\lbrace}{\rbrace}{}{#1}
\DeclarePairedDelimiterXPP{\Ex}[1]{\mathds{E}}{\lbrack}{\rbrack}{}{#1}
\DeclarePairedDelimiterXPP{\Exx}[2]{\mathds{E}_{#1}}{\lbrack}{\rbrack}{}{#2}
\DeclarePairedDelimiterXPP{\Var}[1]{\mathrm{Var}}{\lbrack}{\rbrack}{}{#1}
\DeclarePairedDelimiterXPP{\One}[1]{\mathds{1}}\{\}{}{#1}
\DeclarePairedDelimiterXPP{\norm}[2]{}\lVert\rVert{_{#1}}{#2}

\newcommand\transpsymbol{\text{\upshape\sffamily\bfseries\tiny{T}}}
\providecommand\transp{}
\renewcommand{\transp}[1]{{#1}^{\transpsymbol}}

\declaretheorem[style=plain]{theorem}
\declaretheorem[style=plain,sibling=theorem]{lemma}
\declaretheorem[style=plain,sibling=theorem]{corollary}
\declaretheorem[style=definition,sibling=theorem]{proposition}

\declaretheorem[style=remark,sibling=theorem]{claim}

\declaretheorem[style=definition,numbered=no]{assumptions}

\newenvironment{assumptions*}[1]{%
  \begin{assumptions}
    #1
    \begin{enumerate}[left=0pt]
      \setcounter{enumi}{\theassumption}
      \newcommand{\assume}[1][]{\item\label[assumption]{##1}}
    }{%
      \setcounter{assumption}{\theenumi}
    \end{enumerate}
  \end{assumptions}%
}

\Crefname{assumption}{Assumption}{Assumptions}

\newcommand{\States}{\mathcal{S}}

\newcommand{\CoreStates}{{\mathcal{S}_*}}
\newcommand{\nCoreStates}{m}
\newcommand{\Actions}{\mathcal{A}}
\newcommand{\Span}{\operatorname{span}}
\newcommand{\vPhi}{\vec{\Phi}}
\newcommand{\vPhiCore}{\vec{\Phi}_*}

\newcommand{\vphi}{\vec{\varphi}}

\newcommand{\vdelta}{\vec{\delta}}
\newcommand{\vtheta}{\vec{\theta}}
\newcommand{\veta}{\vec{\eta}}
\newcommand{\vlambda}{\vec{\lambda}}
\newcommand{\vmu}{\vec{\mu}}
\newcommand{\vpi}{\vec{\pi}}
\newcommand{\vrho}{\vec{\rho}}
\newcommand{\vxi}{\vec{\xi}}
\newcommand{\vP}{\vec{P}}
\newcommand{\vE}{\vec{E}}
\newcommand{\vB}{\vec{B}}

\newcommand{\vW}{\vec{W}}
\newcommand{\vWstar}{{\vW_{\!*}}}
\newcommand{\vZ}{\vec{Z}}

\newcommand{\vece}{\vec{e}}

\newcommand{\vecr}{\vec{r}}
\newcommand{\vecv}{\vec{v}}
\newcommand{\vecx}{\vec{x}}
\newcommand{\vecy}{\vec{y}}
\newcommand{\epsapprox}{\epsilon_{\mathrm{approx}}}
\newcommand{\epsopt}{\epsilon_{\mathrm{opt}}}
\newcommand{\epx}{\epsapprox}
\newcommand{\Bee}{\mathcal{B}}
\ifpreprint
  \newcommand{\AppendixName}{Appendix}
\else
  \newcommand{\AppendixName}{Supplementary Material}
\fi
\newcommand{\Appendix}{\hyperref[appendix]{\AppendixName}}
\title{Efficient Planning in Large MDPs with \texorpdfstring{\\}{}Weak Linear Function Approximation}

\author{%
  Roshan Shariff\\
  University of Alberta \& Amii\\
  \mailto{roshan.shariff@ualberta.ca}\\
  \And{}
  Csaba Szepesv\'ari\\
  DeepMind \& University of Alberta \& Amii\\
  \mailto{szepesva@ualberta.ca} \\
}

\hypersetup{%
 pdfauthor={Roshan Shariff \& Csaba Szepesv\'ari},
 pdfsubject={Submitted to the 34th Conference on Neural Information Processing Systems (NeurIPS 2020), Vancouver, Canada.},
 pdfkeywords={}
}

\begin{document}

% Roshan: add title to PDF bookmarks at top level
\phantomsection{}%
\makeatletter\addcontentsline{toc}{chapter}{\@title}\makeatother

{%
  \hypersetup{hidelinks}%
  \maketitle%
}

\begin{abstract}
Large-scale Markov decision processes (MDPs) require planning algorithms with runtime independent of the number of states of the MDP\@.  We consider the planning problem in MDPs using linear value function approximation with only weak requirements: low approximation error for the optimal value function, and a small set of ``core'' states whose features span those of other states. In particular, we make no assumptions about the representability of policies or value functions of non-optimal policies. Our algorithm produces almost-optimal actions for any state using a generative oracle (simulator) for the MDP, while its computation time scales polynomially with the number of features, core states, and actions and the effective horizon.
\end{abstract}

\section{Introduction}%
\label{sec:introduction}

Markov decision processes (MDPs) are a common model for sequential
decision making under uncertainty and have a wide range of applications
\citep[see][for example]{
White93:Apps,
rust96:book,
%FeiSh02:MDPHandbook,
%QiWu07,
%SiBu10:MDPinAI,
%BauRie:11,Puter,
%LeLiu12:RLBook,
%Abuetal15:MDPWireless,
BouDi17:MDPPractice}.  We consider planning in large-scale,
expected discounted total reward MDPs.  Computing an optimal policy in the
discounted setting is known to require ``reading'' all states at least
once~\citep{BlonTsi:00Complexity}. %
As the state space for most interesting applications is intractably
large if not infinite (``Bellman's curse of dimensionality''), it is
common to consider restrictions to the problem that can allow
efficient calculation of near-optimal actions. 
One such relaxation is ``online planning'' --- we ask only for a good
action at a given state when the MDP can be accessed through a
simulator \citep{kearns2002sparse,kolobov2012planning}.
While in this problem the complexity of computing a ``good action''
can be independent of the number of states, the complexity is 
exponential in the planning horizon \citep{kearns2002sparse}.
An alternative idea,
which can be traced back to at least the work of \citet{SchSei85},
is to assume that one has
access to a \emph{feature representation} (that is, a vector of
features for each state) and the planner needs to work well
for those MDPs where the \emph{optimal value function} of the MDP can
be uniformly well approximated over all states by some appropriate weighted
combination of the features.
% Following \citet{SchSei85}, we make the mild assumption that one of the feature components is a nonzero \emph{constant} (say, one).
Since an accurate approximation of the optimal value function
is known to be sufficient to generate near-optimal behavior,%
\footnote{See \cref{prop:policy-performance}; or, for example,
\citet[Lemma~5.17]{szepesvari2001}, \citet[Lemma~5]{kearns2002sparse},
\citet[][Theorem~3.7]{Kall17}.} %
the problem simplifies to producing a good estimate of the
unknown feature weights 
with a computation cost that is independent of the number of states.

In this paper we consider the intersection of these two problem formulations.
More precisely,
our goal is to construct planning algorithms that produce an action
for any given input state, using black-box access to the MDP through a
simulator which takes a state and an action as input and produces a
random next state and immediate reward.  The planner
can also access the feature representation of any state as a
$d$-dimensional feature vector.  Assume that the optimal value
function of the MDP can be uniformly well approximated --- to an
accuracy of $\epx$ --- as a linear combination of the features with
fixed, but unknown, coefficients.  We ask for a randomized planning
algorithm that interacts with the simulator
$\mathrm{poly}(1/\epsilon, d, H, A, \dots)$ times to produce an
action, such that following this action in every state results in an
$O(\epsilon+c \epx)$-optimal policy. Here $H=1/(1-\gamma)$ is the
effective planning horizon for the discount factor $0\le \gamma <1$, which is
used in the definition of the values of policies; $A$ is the number of
actions; and $c>0$ is an error inflation factor that may depend on
$\gamma$, $d$, and $A$;

We call the features ``weak'' as we only require the optimal value
function to be accurately representable by their linear combinations,
in contrast to ``strong'' features that can accurately represent the
value functions of \emph{all} policies; in this
latter case the problem of efficient planning in the presence of a
simulator is known to have a solution, both in the episodic and
discounted settings \citep{YW19,LaSzeGe19}. As pointed out by
\citet{du2019good}, with only weak features, the problem of
efficient planning has not yet been solved.

\paragraph*{Our Contributions} We design a randomized algorithm that
positively answers the challenge posed above under one extra assumption
--- that the feature vectors of all states lie within the convex hull
of the feature vectors of a few selected ``core states'' that the
algorithm is given.  In particular, we show that, as required,
the query-complexity and runtime of our algorithm is polynomial in the
relevant quantities and the number of core states, providing a partial
positive answer to the previously open problem of efficient planning
in the presence of weak features.

To achieve our result, we start from the \emph{approximate linear
  programming} (ALP) approach where the value function is approximated
using the feature vectors.  Following \citet{Lakshminarayanan2018}, we
construct a \emph{relaxed} ALP that drops all the constraints except
at the core states.  In their work, \citeauthor{Lakshminarayanan2018}
gave bounds on the error of the value function that is obtained from
solving this relaxed ALP\@.
% The error bound was given in a weighted $L^1$-norm.
The authors also suggested a way to turn this error bound into an
efficient planning method though without a detailed analysis.  The
main contribution of the present work is to fill this gap, in addition
to simplifying, strengthening and streamlining the earlier results.
In particular, we propose using a randomized saddle-point solver %
\todoc{Look up Theorem 6 of \citep{petrik}}%
that substantially reduces the computational requirements compared to the
procedure hinted at by \citep{Lakshminarayanan2018}. %
\todor{(Other works don't (directly) make this assumption)---\\Along the way, we also remove a limitation of previous saddle-point
algorithms of RL which needed a bound on the unknown feature weight
vector, which in general is not available.}
% \todoc{somewhere this will need to be discussed}%

\paragraph*{Paper Organization}\label{sec:paper-org}

The rest of the paper is organized as follows:
\Cref{sec:background,sec:lp-approach} give background on MDPs and
introduce the linear programming (LP) approach to planning.
\Cref{sec:problem-statement} formally defines the problem.  Then, in
\cref{sec:corelp}, we present the linear program that we start with
and give our first results, bounding the value loss of the policy that
can be read out from optimal solutions of the linear program.
\Cref{sec:algorithm} gives the efficient algorithm to solve the linear
program and our main result.  \Cref{sec:related} discusses related
work. The paper is concluded in \cref{sec:conc}.  The proofs of the
results are moved to \cref{sec:proofs}%
\ifpreprint\else
\ in the \Appendix%
\fi.

\paragraph*{Notation}%
\label{sec:notation}

The set of real numbers is denoted by $\Real$, whereas
$\Real_+=\lbrack 0,\infty \rparen$.  $\Real^d$ denotes the vectors with $d$
dimensions, while the $m \times n$ matrices are $\Real^{m \times n}$.  We
use bold letters for vectors $(\vecr)$ and bold capitals for matrices
$(\vP)$; their elements are written as $r_i$ and $P_{i,j}$ and matrix
rows are $\vP_i$.  For vectors of identical dimension,
$\vec x \le \vec y$ means element-wise comparison: $x_i \le y_i$ for
each index $i$.  The standard basis vector $\vece_i$ has $e_{i,i}=1$
and $e_{i,j}=0$ for $i \neq j$, and the constant 0 or 1 vector is
denoted by $\vec0,\vec1\in\smash{\Real^d}$; their dimensions depends on the context.
All vectors are considered column vectors by default.  The probability
simplex over any finite set $\Actions$ is denoted $\Delta_\Actions \defeq
\set{\vec p \in \smash{\Real_+^{\abs{\Actions}}} \given \norm{1}{p} = 1}$.
For a finite set $\States$ with cardinality $S=\abs{\States}$, we will think of
functions $v:\States\to\Real$ or $\varphi:\States\to\smash{\Real^d}$ as
vectors or matrices, respectively, and use both notations: for
example, $\vec v\in\smash{\Real^S}$, $v_s$, or $v(s)$; and
$\vPhi\in\Real^{S\times d}$, $\vphi_s$, or $\vphi(s)$ where $s\in \States$.
When the domain takes the form $\States \times \Actions$ with respective cardinalities $S$ and $A$,
we use intuitive double indices of the form $sa$, e.g., with $r: \States \times \Actions \to \Real$
we index components of $\vec r$ using the notation $r_{sa}$ (i.e., $r_{sa}=r(s,a)$).
In this case we also write $\vec r\in \Real^{SA}$.

For convenience, an \nameref{sec:notation-index} section is included in the \Appendix.

\subsection{Background}%
\label{sec:background}

A (finite, discounted) Markov Decision Process (MDP) is defined by the
entities $(\States, \Actions, \vec P, \vec r, \gamma)$ where $\States$
and $\Actions$ are finite sets of states and actions, respectively.
Without loss of generality we let $\States = [S]$ and $\Actions = [A]$, using the notation 
$[n]= \{1,\dots,n\}$ for integers $n>0$.
When the process is in state $s\in \States$ and action $a\in \Actions$ is chosen,
a random reward is received with expectation $r_{sa}\in \Real$ and the process transitions
to a new state $s'\in\States$ with probability $P_{sa,s'}$. 
For convenience, we arrange the transition probabilities into a matrix
$\vec P\in\Real^{SA\times S}$ and the rewards into a vector
$\vec r \in \Real^{SA}$. Thus, $\vec P$ is a \emph{row-stochastic matrix} ---
each row $\vec P_{sa}$ for a state $s$ and action $a$ is a valid
probability distribution for the next state.  

For our purposes it will be sufficient to consider stationary policies,
which we will just call \emph{policies}.
A policy $\pi:\States\to\Delta_{\Actions}$ is a function from states
to probability distributions over actions --- we use
$\pi(a|s)$ to denote the probability assigned by $\pi$ to 
action $a$ in state $s$.
Following a policy means that upon visiting state $s$, an action $a\sim\pi(s)$ is chosen at random.
This gives rise to an infinite sequence of states, actions, and corresponding rewards.
The \emph{value} $v_\pi(s)$ of a policy for a process starting at state $s\in \States$ is defined as the
total expected $\gamma$-discounted sum of the rewards incurred: %
%\footnote{In this work, the discount factor $\gamma$ is fixed and
%  specified by the problem --- we write value functions without the
%  $\gamma$ subscript.} %
\begin{align}
  \vecv_{\pi}
  &= \sum_{t=0}^\infty {(\gamma\vP_\pi)}^t \vecr_\pi,
  &\text{where } {[\vecr_{\pi}]}_{s} &= \sum_{a\in\Actions} \pi(a|s)\,r_{sa},
  &\text{and } {[\vP_\pi]}_{s,s'} &= \sum_{a\in\Actions} \pi(a|s) \, P_{sa,s'};
                              \label{eq:value-function}
\end{align}
$\vecr_{\pi} \in \Real^S$ is the expected reward and
$\vP_\pi\in\Real^{S\times S}$ is the state transition matrix.
% Defining $[P_\pi]_{s,s'} = \sum_a \pi(a|s) \, P_{sa,s'}$ and $[r_\pi]_s =
% \sum_a \pi(a|s) \, r_{s,a}$ and using our convention concerning functions, vectors and matrices,
% $\vec v_\pi = \sum_{t=0}^\infty (\gamma\vP_\pi)^t \vec r_\pi$.
A policy $\pi^*$ is called \emph{optimal} if 
$\vecv_{\pi^*}\ge \vecv_\pi$ for every policy $\pi$, where the inequality is component-wise.
Every MDP has an optimal policy, and all optimal policies have
the same value function $\vecv^*$, the \emph{optimal value function}.
Further, there always exist \emph{deterministic} optimal policies,
which concentrate all their probability on a single action for each state.
We will also need $\vec q^*\in \Real^{SA}$, which is defined via $\vec q^* = \vecr+\gamma \vP \vecv^*$.

\subsection{Linear Programming}%
\label{sec:lp-approach}

Our approach to the MDP planning problem is based on the standard
linear programming (LP) formulation; for details, see
\citet[\S6.9]{Puterman}, who swaps the primal and dual problems:
\begin{alignat}{5}
  v^*(s_0)
  &= \min
  &&\,\big\lbrace\,\transp\vece_{\!s_0}\vecv
  &&\,\given[\big] \,\vecv\in\Real^S,
  &\vecr + (\gamma\vP - \vE)\vecv \le{} &\vec0
  &&\;\big\rbrace \label{eq:primal-lp}\tag{Primal LP}
  \\
  &= \max
  &&\,\big\lbrace\,\transp\vmu\vecr
  &&\,\given[\big] \,\vmu\in\Real_+^{SA},\hspace{0.5em}
  &\vece_{\!s_0} + \transp\vmu(\gamma\vP - \vE) ={} &\vec0
  &&\;\big\rbrace. \label{eq:dual-lp}\tag{Dual LP}
\end{alignat}
The matrix $\vec E:\Real^{SA\times S}$ has elements $E_{sa,s} = 1$ and
$E_{sa,s'}=0$ for $s \neq s'$.  It maps vectors from $\Real^S$ to
$\smash{\Real^{SA}}$ by duplicating their elements over all actions:
${[\vec E \vec v]}_{sa} = v_s$ for all $s, a$.  
Both the primal and dual optimization problems above have the same
optimal value: the optimal value of state $s_0$.
The dual variables $\vmu$ are \emph{discounted state-action
  occupancy measures} of a policy $\pi$ starting at state $s_0$:
\begin{align}
  \transp\vmu
  &= \transp\vrho_{\pi} \sum_{t=0}^\infty {(\gamma\tilde{\vP}_\pi)}^t,
  &\text{where }
    {[\vrho_\pi]}_{sa}
  &= \pi(a|s_0) \, e_{\!s,s_0},
  &\text{and } {[\tilde\vP_\pi]}_{sa,s'a'}
  &= \pi(a'|s') {[\vP_\pi]}_{sa,s'}; \label{eq:occupancy-measure}
\end{align} 
$\vrho_\pi \in \Delta_{\States\times\Actions}$ is the initial distribution over
state-action pairs and $\tilde\vP_{\pi} \in \Real^{SA \times SA}$ is
the state-action transition probability matrix; the state transition
matrix $\vP_{\pi}$ is defined
in~\cref{eq:value-function}.  The constraint
$\vece_{\!s_0}+\transp\vmu(\gamma\vP - \vE) = \vec0$ enforces that $\vmu$
has this form, and is therefore generated by some stochastic policy.
Thus the dual problem can be seen as a linear formulation of policy
optimization, where policies are represented by their occupancy
measures and the objective is to maximize expected discounted reward
--- in particular, the policy corresponding to any feasible $\vmu$ can
be obtained by conditioning on the state:
$\pi_{\vmu}(a|s) = \mu_{sa}/\sum_{a'}\mu_{sa'}$ for any state with
non-zero occupancy measure.

The \emph{approximate linear program} (ALP) of \citet{SchSei85} is
obtained from~\cref{eq:primal-lp} by restricting $\vecv$ to lie in
the span of a \emph{feature matrix} $\vPhi\in\Real^{S\times d}$ ---
$\vecv=\vPhi\vtheta$ for $\vtheta\in\Real^d$.  This reduces the number
of variables from $S$ to the feature dimension $d$, but is still
intractable to solve because of the many constraints --- one for each
state-action pair.  The \emph{relaxed ALP} of
\citet{Lakshminarayanan2018} addresses this issue by keeping only a
small number of constraints that are positive linear combinations of
the original constraints; this is the foundation of our approach.

\section{Problem Definition}%
\label{sec:problem-statement}

In the \emph{online MDP planning problem}, 
a randomized planner is given a state of the MDP $s_0\in \States$ as input,
and needs to return an action \citep[e.g.,][]{kearns2002sparse,FeCa14}.
Letting $\pi(a|s_0)$ denote the probability that action $a$ is returned for input $s_0$,
the planner's \emph{value loss} at state $s$ is defined as $v^*(s)-v_\pi(s)$.
The goal is to design planning algorithms with a small value loss for
\emph{every} start state $s$ regardless of the MDP\@.

For large MDPs, we want the computation time to be
independent of the number of states $S$ and depend polynomially on 
the number of actions and the ``planning horizon'' $H=1/(1-\gamma)$.
To make this even remotely possible, 
we assume that the planner has access to a suitable 
feature map $\vphi:\States \to \Real^d$ (\cref{ass:function-approx} below),
is given a suitable set of ``core states'' (\cref{ass:core-states}),
and can access a simulator of the MDP (\cref{ass:simulator}).%
\footnote{\label{note:lower-bound}%
  With no additional assumptions on the MDP, any online planning
  algorithm implementing an $\epsilon$-suboptimal policy may need up
  to $\smash{\Omega({(1/\epsilon)}^{H-1})}$ simulator queries to find each
  action; this is exponential in the planning horizon $H$ for any
  constant $\epsilon > 0$ --- see \citet[Theorem~2]{kearns2002sparse},
  noting that their $H$ is different from ours.} %
Further, the planner is only required to perform well if the optimal value function
lies uniformly close to the span of the features, which means that
\begin{align}
  \epsapprox
  &\defeq \adjustlimits \inf_{\vtheta\in \Real^d}\max_{s\in \States}\;\abs{v^*(s) - \transp\vphi_s\vtheta}
\label{eq:epsapprox}
\end{align}
is small. To be precise, the planner's value loss is allowed to
degrade with $\epsapprox$.  Note that the class of MDPs where
$\epsapprox$ is small for a given feature map is a strict superset of
those which are nearly linear up to the error $\epsapprox$ in the
sense of \citet{jin2019provably}.  Hence, we call the features
\emph{weak} because we only require that $\epsapprox$ as defined
in~\cref{eq:epsapprox} be small.  \citet{du2019good} posed the open
problem of designing efficient online planning algorithms under this
condition.
%As we shall see, while our work addresses this problem, we will not be able to fully resolve it. \todoc{does the planner to know $\epsapprox$? hopefully not}

\begin{assumptions*}{For the convenience of the reader, we now restate our assumptions in a concise form:}
\assume[ass:function-approx] \emph{Features:}
  The planner can access $\vphi(s)\in\Real^d$ for any state $s\in
  \States$.  Further, there is some
  $\veta\in\Real^d$ such that $\transp\vphi_s\veta = 1$ for all
  $s$ --- this can be ensured easily by adding a
  ``bias'' feature that is always 1.
\assume[ass:core-states] \emph{Core States:}
  There is a set of core states $\CoreStates\subset\States$ (with
  $\abs{\CoreStates} = \nCoreStates$) that are available to the
  algorithm, and the feature vector of every other state can be written
  as a positive linear combination of the core state features: $\vPhi =
  \vZ\vPhiCore$ for some non-negative matrix
  $\vZ\in\Real_+^{S\times\nCoreStates}$, where $\vPhi\in\Real^{S\times
    d}$ and $\vPhiCore\in\Real^{\nCoreStates\times d}$ consist of the
  stacked feature vectors for all states and the core states,
  respectively.
\assume[ass:simulator] \emph{Simulator:}
  The planner can call a randomized function $\textsc{Simulate}(s, a)$
  that returns $s' \sim \vP_{sa}$ and a reward $\hat r$ with $\Ex{\hat r}
  = r_{sa}$.  For simplicity, we assume $\abs{\hat r} \le 1$.
\end{assumptions*}

We do not need to explicitly assume that the feature vectors of all
states lie within the convex hull of the core state features --- that
is a consequence of \cref{ass:function-approx,ass:core-states}
(specifically, $\vec1\in\Span\vPhi$).%
\footnote{We have
  $\vZ\vec1 = \vZ\vPhiCore\veta = \vPhi\veta = \vec1$ for some
  $\veta\in\Real^d$, so the rows of $\vZ$ must sum to one.} %
Under \cref{ass:core-states}, each of the core states is a ``soft
state aggregation'' \citep{SinghSoftAggregation1995} that respects
the feature representation.  \citet{Zanette2019}
impose a similar requirement for core states, observing that
``anchoring'' the value function at states with ``extreme'' feature
representations (i.e., on the boundary of the convex hull of feature
vectors) results in the values of other states being
\emph{interpolated} (not \emph{extrapolated}) from the values of the
core states using their respective feature representations --- this is
sufficient to accurately deduce the values of all states when $\epx$
is small \citep{YW19}.

Without assuming extra structure, finding a set of core states
requires checking the feature vectors of all states, which is
intractable in large MDPs. It remains an interesting question what
extra structure would make it possible to discover near-minimal core
sets with an effort independent of the size of the MDP\@.  We also
note that for some feature maps the size of the minimal core set can
be as large as the number of states $S$.  Since the run time of our
algorithm depends on the size of the core set, one should obviously
avoid such feature maps. It remains an intriguing question whether
requiring a small core set is necessary for efficient planning.%

The algorithm we design can work with weaker assumptions --- it only
accesses $\vphi(s)$ for the planning state $s_0$, the core states, and for the
next state $s'$ produced by the simulator: $(s', \hat r) \gets
\textsc{Simulate}(s, a)$, which it only queries with $s\in\set{s_0}\cup\CoreStates$.
% (which we may assume have distinct feature vectors, without
% loss of generality).
We also note in passing that the $\vZ$ matrix of
\cref{ass:core-states} is not used by the algorithm and need not be
known; it need only exist.  Finally, our results also hold if the core
states are so-called ``meta-states'' --- probability distributions
over underlying MDP states; the simulator must then be able to sample
states according to these probability distributions.

\section{CoreLP --- A Linear Program for Planning with Core States}%
\label{sec:corelp}

Throughout this section, we will use $s_0\in\States$ to
refer to the current planning state --- our goal is to output a
random close-to-optimal action $a \sim \pi(s_0)$.
Consider the following result, which follows immediately from the
well-known ``performance difference lemma''
\citep[Lemma~6.1]{KakadeLangford2002}:
\begin{proposition}\label{prop:policy-performance}
  Let $\pi$ be an arbitrary policy. Then,
  \begin{align*}
    \max_{s\in \States} v^*(s)-v_\pi(s)
    &\le \frac1{1-\gamma} \max_{s_0\in \States} \Exx{a\sim\pi(s_0)}{v^*(s_0) - q^*(s_0,a)}.
  \end{align*}
  \vspace{0pt}
\end{proposition}
In light of this, we will design a randomized planning algorithm that
guarantees $\Exx{a\sim\pi(s_0)}{q^*(s_0, a)} \approx v^*(s_0)$ for any
input state $s_0$.
Our approach stems from the relaxed linear program of
\citet{Lakshminarayanan2018} --- more precisely, we start with the ALP
but keep only the constraints corresponding to the actions of the core
states and the current planning state (see \cref{sec:lp-approach}).
We then construct the corresponding dual LP, to which we add
another constraint that allows us to ``read out'' an action
distribution from the values of the dual variables.

Recall that $\CoreStates = \set{s_1,\dotsc,s_m}$ is the set of core
states and define $\States_+$ as the sequence
$(s_0,s_1,\dotsc,s_m)$. Note that $s_0$ is always the first state in
$\States_+$ but may appear again if it is also a core state.  For each
of the $1+m$ states in $\States_+$, we will select the $A$ constraints
in the ALP corresponding to the state-action pairs
$(s_i, a) \in \States_+\times\Actions$ --- a total of $(1+m)A$
constraints.  Unlike the ALP, our linear program \textsc{CoreLP} is
based on the~\cref{eq:dual-lp} of \cref{sec:lp-approach}; the name
refers to the \textsc{core} states and features that, respectively,
\textsc{constrain} and \textsc{relax} it.

\begin{restatable}[CoreLP]{theorem}{TheoremCoreLP}\label{thm:corelp}
  Suppose \cref{ass:function-approx,ass:core-states} hold,
  \(s_0\in \States\), \(\vphi_0 \defeq \vphi_{s_0}\),
  \(\vW\in\set{0,1}^{(1+m)A\times SA}\) has rows \({[\vW_{s_i a}]}_{s_i\in\States_+,a\in\Actions} =
  \vece_{s_i a}\),
  and
  \(\Lambda\defeq\set{\vlambda\in\Real_+^{(1+m)A} \given \sum_{a\in\Actions}\lambda_{s_0a} = 1}\).
  Define
  \begin{alignat}{5}
    V^\dagger
    &= \max
    &&\,\big\lbrace\,\transp\vlambda\vW\vecr
    &&\,\given[\big] \,\vlambda\in\Lambda,\;\,
    & \transp\vphi_0 + \transp\vlambda\vW(\gamma\vP - \vE)\vPhi ={} &\vec0
    &\;\big\rbrace . \label{eq:corelp}\tag{CoreLP}
  \end{alignat}
  Let \(\vlambda^\dagger \in \Lambda\)
  be a maximizer of~\cref{eq:corelp} and let
  \(\vpi^\dagger\in\Delta_\Actions\) be given by \(\pi^\dagger(a) =
  \lambda^\dagger_{s_0a}\).  Then
  \begin{align*}
    \abs{V^\dagger - v^*(s_0)} &\le \frac{10\gamma\epsapprox}{1-\gamma}, &
    v^*(s_0) - \sum_{a\in\Actions}{\pi^\dagger(a)\, q^*(s_0, a)} &\le \frac{20\gamma\epsapprox}{1-\gamma}.
  \end{align*}
\end{restatable}

This bound matches up to constant factors (and improves by a $\gamma$
factor) the landmark result of \citet{ALP} for the approximation error
of the ALP (defined in \cref{sec:lp-approach}).  In other words, core
states satisfying \cref{ass:core-states} lead to essentially no
additional error in the solution of~\cref{eq:corelp} compared to the
ALP --- this was pointed out by \citet{Lakshminarayanan2018}, whose
result we improve upon in \cref{thm:lralp} (\cref{sec:lralp-proofs}).
The \lcnamecref{thm:corelp} also implies that the linear program is
both bounded and feasible, meaning its value is not $\pm\infty$; this
is an important consideration when relaxing the ALP
\citep{bas2019faster}.  We present the detailed proof in
\cref{sec:corelp-proof}%
\ifpreprint\else
\ of the \Appendix%
\fi.

The feature matrix $\vPhi$, which in the ALP constrains the value
functions of~\cref{eq:primal-lp}, instead \emph{relaxes} the
constraint in~\cref{eq:dual-lp} to be
$[\transp\vece_{\!s_0} + \transp\vmu(\gamma\vP - \vE)]\vPhi =
\vec0$.  As a result, $\vmu$ may no longer be a discounted
state-action occupancy distribution, although it behaves like one with
respect to expectations of functions in the span of $\vPhi$ --- using
the notation of \cref{sec:lp-approach,sec:background}, any feasible
$\vmu$ satisfies
$\transp\vmu\vE \vec f =
\transp\vece_{\!s_0}\!\sum_{t=0}^\infty{(\gamma\vP_\pi)}^t \vec f$ for some
policy $\pi$ and any $\vec f = \vPhi\vtheta$; compare this
with~\cref{eq:occupancy-measure}.

Conversely, the $\vW$
matrix \emph{constrains} $\vmu$ to be non-zero only on the actions of
core states and the current planning state:
$\vmu = \transp\vlambda\vW$.  The discounted visits to all other
states are ``soft-aggregated'' as positive linear combinations of the
core states, as discussed in \cref{sec:problem-statement}.  Such
aggregation is acceptable because
\begin{enumerate*}[label=\emph{(\roman*)},itemjoin={{; }},itemjoin*={{; and }}]
\item the above relaxation means $\vmu$ only needs to be accurate for
  functions in the span of $\vPhi$
\item \cref{ass:core-states} ensures that the features of all states
  lie in the convex hull of the core state features.
\end{enumerate*}
Thus the simultaneous constraint and relaxation complement each other,
incurring the same $O(\epx/(1-\gamma))$ error as the ALP which also
restricts value functions to the span of $\vPhi$.

Significantly, this \lcnamecref{thm:corelp} also specifies how to
select an action that achieves the promised value for the planning
state $s_0$.  This is made possible by
\begin{enumerate*}[label={\emph{(\roman*)}},itemjoin={{; }},itemjoin*={{; and }}]
\item adding $s_0$ to the set of core states
\item requiring (via the definition of $\Lambda$) that
  $\mu_{s_0a} \equiv \lambda_{s_0a}$ sum to one.
\end{enumerate*}
This last constraint forces $\vmu$ to directly represent the action
probabilities at the planning state, not indirectly by being
aggregated as linear combinations of the core state actions.  As a
result, a solution $\vmu \equiv \transp\vlambda\vW$
of~\cref{eq:corelp} yields an almost-optimal action distribution
$\pi(a|s_0) = \mu_{s_0a}$ for the planning state $s_0$.

Unfortunately, the soft state aggregation which
makes~\cref{eq:corelp} tractable to solve (as in
\cref{sec:algorithm}) comes at a price --- $\vmu$ directly encodes
only an action distribution for the current planning state, not a
policy for other states (unlike the original~\cref{eq:dual-lp} of
\cref{sec:lp-approach}).  We believe that solving a separate
optimization problem for each planning state is unavoidable without
restricting ourselves to a compactly representable class of policies;
see \cref{sec:related} for a discussion of such approaches.

As a final remark, the value loss of the policy resulting from
\cref{thm:corelp} is $O(\gamma\epx/{(1-\gamma)}^2)$. 
Here, an extra $1/(1-\gamma)$ factor is incurred 
in \cref{prop:policy-performance}, while the other $1/(1-\gamma)$ factor is incurred 
in \cref{thm:corelp}.
This is similar to the bounds obtained in previous works
\citep[e.g.,][]{ALP,YW19,LaSzeGe19}. \todoc{I commented out the comment on \citet{SinghYee94},
I don't think this is relevant. The result of \citet{SinghYee94} is 
more similar to \cref{prop:policy-performance}: When used with ADP, this result
also gives an $1/{(1-\gamma)}^2$ result because the action-value function estimate will incur
and additional $1/(1-\gamma)$ factor.
}
\if0
Just as with other ALP-based
approaches \citep{ALP}, this has an extra $1/(1-\gamma)$ factor
compared to the result of \citet{SinghYee94} for greedy policies
derived from approximate value functions --- it is an interesting
question whether this is an artifact of the analysis, due to function
approximation, or due to the linear programming formulation; and
whether it can be avoided.
\fi

\section{CoreStoMP --- A Stochastic Saddle-Point Algorithm}%
\label{sec:algorithm}

Having formulated the planning problem as a linear program with few
variables and constraints, the remaining issue is that the constraints
still involve quantities of the form $\vW\vP\vPhi$, which
cannot be calculated exactly in time independent of $S$. However,
since these are actually expectations, the simulator can be used to
estimate them. One possibility would be to estimate
$\vP_{sa}\vPhi \in \Real^d$ for the initial and core states and use a
plug-in estimator --- often called \emph{sample average
  approximation}.  Instead, we pursue the \emph{stochastic
  approximation} approach --- using well-known first-order
optimization methods to directly solve~\cref{eq:corelp} by using
the simulator to produce stochastic estimates of gradients that are
intractable to compute exactly \citep{JuditskyOptMLCh5,JuditskyOptMLCh6}.
This optimization-based approach is attractive to us because the
resulting algorithm, by design, is \emph{incremental} and
\emph{anytime} --- the quality of the solution steadily improves if
the algorithm is given more time.%
\todoc{compare the two approaches}

We first rewrite~\cref{eq:corelp} as an unconstrained ``saddle
point'' problem, retaining only the constraint introduced by the
definition of $\Lambda$:
\begin{align}
  V^\dagger &= 
  \adjustlimits \max_{\vlambda\in\Lambda} \min_{\vtheta\in\Real^d} \;\brck[\big]{\,f(\vlambda, \vtheta)
  \defeq \transp\vlambda\vW\vecr + \transp\vphi_0\vtheta
    + \transp\vlambda \vW(\gamma\vP - \vE)\vPhi\vtheta\,}.
  \tag{Saddle CoreLP}\label{eq:corelp-saddle}
\end{align}
To be able to use first-order methods, we calculate the gradients of $f$:
\begin{align}
  f_{\vlambda}(\vtheta) &\defeq \nabla_{\vlambda} f(\vlambda, \vtheta)
  = \vW(\vecr + (\gamma\vP - \vE)\vPhi \vtheta), \label{eq:flambda}\\
  f_{\vtheta}(\vlambda) &\defeq \nabla_{\vtheta} f(\vlambda, \vtheta)
  = \transp\vphi_0 + \transp\vlambda \vW(\gamma\vP - \vE)\vPhi. \label{eq:ftheta}
\end{align}
Note that $f$ is \emph{bilinear:} its gradient with respect to
$\vtheta$ only depends linearly on $\vlambda$, and vice versa.  The
transition probabilities, which present the major computational challenge,
appear only through the matrix $\vB \defeq \vW(\gamma\vP - \vE)\vPhi$,
whose rows correspond to state-action pairs
$(s,a)\in\States_+\times\Actions$.  Each row
$\vB_{sa} = \gamma\vP_{sa}\vPhi - \transp\vphi_s$ is the
(discounted) expected change in the feature vector when taking action $a$ in state
$s$, which suggests how to estimate it using the simulator ---
define $\Delta\vphi(s, s') \defeq \gamma\vphi_{s'} - \vphi_s$ and
sample $s' \sim \vP_{sa}$; then
$\Delta\tilde\vphi \defeq \Delta\vphi(s, s')$ is an unbiased estimator
of $\vB_{sa}$.  The construction of the matrix $\vec W$ ensures that
$s$ is either the current state or one of the core states. Further, we
only use $s'$ through its feature representation $\vphi_{s'}$.  Putting all
this together, our gradient estimates are:
\begin{align}
  {[\hat{f}_{\vlambda}(\vtheta)]}_{sa} &\defeq \hat{r} + \transp{\Delta\vphi(s, s')}\vtheta,
  &\forall s\in\States_+, a\in\Actions,\quad\text{where } (\hat r, s')\sim\textsc{Simulate}(s,a),
    \label{eq:flambda-est} \\
  \hat{f}_{\vtheta}(\vlambda) &\defeq \transp\vphi_0 + \norm{1}{\vlambda}\Delta\vphi(s,s'),
  &\text{where } (s,a)\sim\vlambda/\norm{1}{\vlambda} \text{ and } s'\sim\textsc{Simulate}(s,a).
  \label{eq:ftheta-est}
\end{align}
Sampling both gradients requires a total of $1 + (1+m)A$
queries of the simulator and an additional $O(d \nCoreStates A)$
computation time.  By slightly abusing notation, we will use
$\vec \xi \sim \hat{f}_{\vtheta}(\vlambda)$
(and $\vrho \sim \hat{f}_{\vlambda}(\vtheta)$) to denote a random
$d$-dimensional (resp., $(1+m)A$-dimensional) vector taken from the
distribution of $\hat{f}_{\vtheta}(\vlambda)$ (resp., that of
$\hat{f}_{\vlambda}(\vtheta)$) as defined above.  Finally, we remark
in passing that the gradient estimate
${[\hat{f}_{\vlambda}(\vtheta)]}_{sa}$ is the ``temporal
difference error'' \citep{SuttonTD} of the value function
$\vPhi\vtheta$ at state $s$ with action $a$.

We use these gradient estimates with the Stochastic Mirror-Prox
algorithm of \citet{JuditskyStochMP2011}.  Instantiating the algorithm
requires several choices --- for the dual variables
$\vlambda\in\Real_+^{\smash{(1+m)A}}$, we use the $1$-norm and the
``unnormalized negentropy'' regularizer; for the primal variables
$\vtheta\in\Real^{\smash{d}}$ we use the norm
$\norm{}{\vtheta} = \norm{2}{\vPhiCore\vtheta}$ and the regularizer
$\norm{}{\vtheta}^2/2$.  The result is \Cref{alg:corestomp} (\textsc{CoreStoMP}).

\begin{algorithm}[t]
  \begin{algorithmic}

    \vspace{0.25\baselineskip}

    \State{%
      \textbf{Parameters: } $T, B, \eta$%
    }

    \vspace{0.25\baselineskip}

    \State{%
      \textbf{Initialization: } $\vtheta_0 \gets \vec0\in\Real^d, \;
      {[\lambda_0]}_{s_0a} \gets 1/A, \;
      {[\lambda_0]}_{sa}\gets \gamma/((1-\gamma)mA) \quad\;\,
      \forall\,s\in\CoreStates,\,a\in\Actions$%
    }

    \vspace{0.25\baselineskip}

    \For{$\tau = 1, 2, \dotsc, T$}
    \vspace{0.25\baselineskip}
    \State{%
      $\begin{alignedat}{3}
        (\vtheta'_\tau, \vlambda'_\tau)
        &\gets \Call{ProxUpdate}{B, \eta, (\vtheta_{\tau - 1}, \vlambda_{\tau-1}),
          (\vxi, \vrho)} &\quad\text{where}\;
        \vxi&\sim\hat{f}_{\vtheta}(\vlambda_{\tau-1}),& \vrho&\sim\hat{f}_{\vlambda}(\vtheta_{\tau-1}) \\
        (\vtheta_\tau, \vlambda_\tau)
        &\gets \Call{ProxUpdate}{B, \eta, (\vtheta_{\tau - 1}, \vlambda_{\tau-1}),
          (\vxi', \vrho')} &\quad\text{where}\;
        \vxi'&\sim\hat{f}_{\vtheta}(\vlambda'_\tau), &\vrho'&\sim\hat{f}_{\vlambda}(\vtheta'_\tau)
      \end{alignedat}$%
    }
    \vspace{0.25\baselineskip}
    \EndFor{}

    \vspace{0.25\baselineskip}
    \State{%
      \Return{%
        % $\inv{\brck*{\textstyle\sum_{\tau=1}^T\eta_\tau}}\textstyle\sum_{\tau=1}^T\eta_\tau
        % \vlambda_\tau$
        $\textstyle\paren[\big]{\sum_{\tau=1}^T \!\vlambda_\tau}/T$%
      }%
    }

    \vspace{0.5\baselineskip}
    \hrule
    \vspace{0.5\baselineskip}

    \Function{ProxUpdate}{$B, \eta, (\vtheta, \vlambda), (\vxi, \vrho)$}
    \vspace{0.25\baselineskip}
    \State{%
      $\begin{alignedat}{4}
        &\tilde\vtheta &&\gets \vtheta - \eta\vxi \\
        &\vtheta' &&\gets \tilde\vtheta / \max\set{1, \norm{2}{\vPhiCore\vtheta}/B} \\
        &\tilde{\vlambda} &&\gets \exp(\log\vlambda + \eta\vrho)\\
        &\vlambda'_{s_0} &&\gets \tilde\vlambda_{s_0}/\norm{1}{\tilde\vlambda_{s_0}}
        &\qquad\text{where }\tilde\vlambda_{s_0} &\defeq
        \brck{\tilde\lambda_{s_0a}}_{a\in\Actions} \text{ and similarly for } \vlambda'.\\
        &\vlambda'_* &&\gets
        (\gamma/(1-\gamma))\tilde\vlambda_*/\norm{1}{\tilde\vlambda_*}
        &\qquad\text{where }\tilde\vlambda_* &\defeq
        \brck{\tilde\lambda_{sa}}_{s\in\CoreStates, \, a\in\Actions} \text{ and similarly for } \vlambda'. \\
      \end{alignedat}$%
    }
    \vspace{0.25\baselineskip}
    \State{%
      \Return{%
        $(\vtheta', \vlambda')$%
      }%
    }
    \vspace{0.25\baselineskip}
    \EndFunction{}
    \vspace{0.25\baselineskip}
  \end{algorithmic}
  \caption{\label{alg:corestomp}\textsc{CoreStoMP}: Stochastic
    Mirror-Prox for Planning with Core States}
\end{algorithm}

\begin{restatable}[CoreStoMP]{theorem}{TheoremCoreStoMP}\label{thm:corestomp}
  Suppose \cref{ass:function-approx,ass:core-states,,ass:simulator}
  hold, and define
  \begin{align*}
    B &\defeq \frac{(9/8)\sqrt{m}}{1-\gamma},
    & C &\defeq \frac{(9/4)\sqrt{m(1+2\log A + 2\gamma\log m)}}{{(1-\gamma)}^2}.
  \end{align*}
  Let \(\hat\vlambda\) be the result of running \cref{alg:corestomp} for
  $T$ iterations with the parameter $B$ and the step size
  \(\eta = \inv{C}\sqrt{2/7T}\), which requires
  \(2T(1+(1+m)A)\) simulator queries.  Define \(\hat\vpi\in\Delta_{\Actions}\) by
  \(\hat\pi(a) = \hat\lambda_{s_0a}\) (as in \cref{thm:corelp}) and \(a \sim \hat\vpi\).  Then
  \begin{align*}
    v^*(s_0) - \Ex{q^*(s_0, a)}
    &\le \frac{32\epsapprox}{1-\gamma} +
      \frac{21}{2{(1-\gamma)}^2} \sqrt{\frac{3m(1+2\log A + 2\gamma\log m)}{T}}.
  \end{align*}
\end{restatable}

Note that the expectation on the left-hand side is both for the
randomness of the algorithm and the action $a$. While the bound does
not have a direct dependence on the dimension of the features, the
number of core states, $m$, must exceed the rank of $\vPhi$.  It is
notable that the approximation error does not get inflated by a
rank-related quantity, as one would expect in the worst-case
\citep{LaSzeGe19}; this is due to \cref{ass:core-states}.  The
increase in the leading term of the approximation error compared to
\cref{thm:corelp} is because of the need to bound the domain of
$\vtheta$ by $B$; it remains for future work to avoid this necessity.
Altogether, \cref{alg:corestomp} gives the following positive result
for the online planning problem for MDPs.
\begin{corollary}
  Under \cref{ass:function-approx,ass:core-states,,ass:simulator},
  \cref{alg:corestomp} is a randomized planning algorithm that, for any
  \(\epsilon>0\), uses \(O(m^2A(1 + \log A + \gamma\log m)/\epsilon^2)\) simulator queries
  and \(\mathrm{poly}(d,A,m,1/\epsilon)\) computation to output an
  action.  Following this action in every state gives a stochastic
  policy with value loss at most
  \(O(\epsapprox/{(1-\gamma)}^2 + \epsilon/{(1-\gamma)}^3)\).
\end{corollary}

\section{Related Work}%
\label{sec:related}

The online MDP planning problem formulation we adopt --- where the
planner is given an input state and asked to produce a
close-to-optimal action using a generative model of the MDP as a
subroutine --- was proposed by \citet{kearns2002sparse} as an
alternative to requiring a compact, structured representation of the
MDP\@.  Their approach, also adopted by \citet{kocsis2006bandit} for
their UCT algorithm, is to build a (sparse) look-ahead tree.
Generally, the problem is that the tree needs to be sufficiently deep
and the branching factor can be as large as the number of actions,
which leads to an exponential blow-up as a function of the planning
horizon (see~\vref{note:lower-bound}).
The focus is thus to characterize those MDPs where the planning time
can be kept polynomial in the effective horizon \citep{Mu14,FeCa14}.
 
\paragraph*{Planning with Feature Representations}

The broader context of this work is the problem posed by the recent
paper of \citet{du2019good}, which asks whether ``good features'' (or
representation) are sufficient in various RL contexts --- including
efficient online planning in large MDPs with a generative model.
%The most restricting assumption on ``good features'' is that they the resulting linear subspace of value function is closed under the policy evaluation operator for \emph{any} policy,
%or, alternatively, the inherent one-step Bellman error of the subspace is zero \citep{anszemu:mlj07}.
%If this is required to hold for all action-value functions (instead of state value functions), the condition is by and large equivalent to assuming that the transition kernel takes a special, linear form \citep{JYW19}.
%Sticking to value functions, 
%slightly less demanding is that the value-function of any policy lies in the said linear subspace.
%An alternative is to require that the linear subspace is closed under the application of Bellman optimality operator,
%which is also phrased as requiring that the inherent Bellman error \citep{munos2008} of the subspace is zero.
%The relaxation of this is that the optimal value function lies in the subspace.
%All of these conditions can be relaxed to hold up to a fixed error, where the error is usually measured in the maximum-norm.
Their main (negative) result states that even when 
the features are good enough to 
represent the action-value functions of all policies up to a uniform error of $\epx$,
a planning algorithm that is required to produce an $O(\epx)$-optimal
policy needs to check at least $2^H$ states in some $H$-horizon episodic problems.
\Citet{LaSzeGe19} along with \citet{DV19} point out that if the feature space is $d$-dimensional, the exponential blowup with the planning horizon can be avoided if the policy only needs to be $\smash{O(\epx\sqrt{d}H^2)}$-optimal (where the horizon is $H=1/(1-\gamma)$, as their results are for discounted problems). They also describe an instance of approximate policy iteration that achieves this bound with $\smash{\cramped{\tilde{O}(d/(\epx^2{(1-\gamma)}^4))}}$ queries, where $\tilde{O}$ hides logarithmic factors.

For the finite-horizon setting, \citeauthor{du2019good} also present a
positive result \citep[Theorem~C.1]{du2019good} for the case when a
simulator of the environment is available and the optimal action-value
function can be represented with no error (i.e., $\epx = 0$).
The proposed method is a randomized algorithm --- an instance of fitted value iteration.
In addition to the usual inputs, the algorithm also takes as input $\delta$, a target failure probability. 
The algorithm returns an optimal policy with probability $1-\delta$,
while issuing at most $\text{poly}(d,H,\log(1/\delta),1/\rho)$ queries
to the simulator, where $\rho$ is the minimum action-value gap that
also needs to be known to the algorithm. %
\todoc{is this knowledge really needed?}%
The algorithm also relies on an oracle to construct a ``core set'' of $d$ state-action pairs for each stage of the $H$-horizon problem whose feature vectors form a \emph{barycentric spanner} of the set of all feature vectors at that stage. The idea of the algorithm is to construct a policy backwards by estimating the action value functions via interpolation: In each stage, the action-value of each member of the core set is estimated by using sufficiently many rollouts using the policy constructed for the further stages. The estimated values are used with barycentric interpolation to produce values for all the other state-action pairs.

For the same finite-horizon setting but allowing for an $\epx$ error in approximating the optimal action-value function, \citet{Zanette2019} describe a similar algorithm.
The main difference is that their algorithm uses the estimated values
in a Monte Carlo procedure in place of policy roll-outs. %
\todoc{I bet this is worse than the rollouts, no?}%
They also propose using a core set (which they call the anchor points) and a similar barycentric extrapolation procedure.
Unfortunately, the errors propagate multiplicatively between the
stages and thus, in the worst case, the error can be as large as $C^H$
where $C>1$ depends on the choice of the features. \Citet{LaSzeGe19}
show that $1\le C \le \sqrt{d}$; we note in passing that ``state aggregation'' gives rise to $C=1$.

A number of authors have studied the problem of learning and planning with exact linear optimal action-value function under various extra conditions.
Positive results have been shown for deterministic MDPs
\citep{WeVR13}, the so-called ``low Bellman rank'' MDPs
\citep{JiKr17}, and under a specific low variance and large gap condition
\citep{DuLuo19}.  \Citet{YW19} assume the transition matrix has a linear structure and
also use least-squares regression with data from a pre-selected collection of anchor state/action pairs.
Their assumption --- the same as ours --- is that the features of all state-action pairs can be
written as convex combinations of the anchoring features. %
\todoc{actually, they don't need this; barycentric spanners would be enough for them, too.}%
They show that their algorithm needs 
at most $\mathrm{poly}(d,1/(1-\gamma),\log(1/\delta),m)$ queries, where $m$ is the number of anchor points.
Their bound scales linearly with $H^7$ where $H=1/(1-\gamma)$. 
Their result also applies to the ``misspecified'' case when the linear structure is only true up to a fixed error.
In contrast to these results, we do not assume that the transition matrix has special structure; we make the weaker assumption that the optimal value function lies close to the span of the features.

\todoc{
\citet{anszemu:mlj07} and \citet{munos2008} prove results in terms of concentrability, a form of distribution mismatch
(similar to \citet{KakadeLangford2002}).
\citet{Chen2019-fo} argue for the necessity of these mismatch coefficients.
The calculations in \citet{abbasi2019politex} seem to show that if action-value functions of all policies can be estimated with a uniform error of $\epsilon$, $n$ iterations of their Politex algorithm gives rise to a policy with suboptimality of at most $\epsilon+\tilde{O}(1/\sqrt{n})$.
}

%
%
%The central question is whether a computationally (and query efficient) planning algorithm can be devised under any of the aforementioned conditions. 
%A line of work studied approximate dynamic programming methods 
%
%\citep{munos2008} describes a value-iteration method 
%where a Monte-Carlo approximation is used on the Bellman optimality 
% randomly chosen ``anchor states''
%
%\todoc[inline]{it may be worthwhile to think later about the finite horizon case which may allow more transparent results and/or more insight.}
%
%Naturally, 
%
% or, alternatively, the so-called inherent Bellman error of the linear space is zero (or small).
%, the meaning of ``good features'' could be either that the features can capture the value function

\paragraph*{Approximate Linear Programming}

The narrower context of the present work is the so-called approximate linear programming (ALP) approach to approximate planning in large MDPs, described in~\cref{sec:lp-approach}.
%
%\cite{
%schuurmans,
%gkp,
%ALP,
%CS,
%kveton2004heuristic,
%petrik,
%SALP,
%fs,
%npalp,
%BhatFaMo12:SALPNP,
%abbasi}.
The seminal work of \citet{ALP} showed that the ALP solution's error,
compared to the optimal value function, is within a constant factor (involving $1/(1-\gamma)$) 
of
the best approximation error achievable by linear combinations of the
given features.  Unfortunately, as discussed earlier, the ALP has too
many constraints to be tractable for large MDPs.  Most
subsequent work is therefore aimed at designing methods that keep
the approximation guarantees without having to enumerate all the
constraints.  \Citet{schuurmans} and \citet{gkp} propose using
``constraint generation'' for problems with additional structure
(i.e., factorized transitions), while \citet{CS} propose randomly
generating a subset of constraints from some a priori fixed
distribution.  All these methods require computation time that depends
on uncontrolled quantities, such as the so-called induced width of a
cost-network \citep{gkp}, or the discrepancy between the sampling
distribution and the (unknown) optimal stationary distribution
\citep{CS}.
%(as \citeauthor{gkp} note, the number of constraints generated can be bounded as an exponential function of
%the induced width of a so-called cost-network, which is in general hard to control).
The fundamental difficulty is that when too many constraints are
dropped, the linear program may become unbounded.  To protect against
this, \citet{CS} add an extra constraint on the optimization
variables, but their bound then degrades to the \emph{worst}
approximation error over this constraint set. 

\Citet{petrik} demonstrate that the $1/(1-\gamma)$ blow-up of the error in
the bound of \citet{ALP} can be tight.  They also propose techniques
to avoid it --- one of them is to add extra constraints induced by
short action sequences; another is to replace the hard constraints
in the ALP with smooth ones with an associated Lagrange multiplier.
 \Citet{SALP} propose a specific way to choose the
Lagrange multiplier,
for which they also obtain error bounds and demonstrate improved
empirical behavior.
However, as they build on the work of \citet{CS}, their results inherit
the limitations of this latter work: the large number of constraints.
\Citet{BhatFaMo12:SALPNP} extend the work
of \citet{SALP} to nonparametric function approximation.
\Citet{Lakshminarayanan2018} depart from constraint
sampling and consider the error induced by linearly combining constraints.
\Citet{petrik}, in addition to the above mentioned contributions, also give
error bounds for the ALP obtained
by replacing the transition matrix with a sample-average estimate. \todoc{and we should compare to what can be obtained with this technique vs. using mirror-prox.}

\paragraph*{The Dual Linear Program}
A parallel line of research aims to solve (an approximation of)
the~\cref{eq:dual-lp} optimization problem, in contrast to the
aforementioned work focusing on~\cref{eq:primal-lp}.  Recall from
\cref{sec:lp-approach} that the dual variables $\vmu$ are occupancy
distributions over state-action pairs generated by policies --- a
common theme in these approaches is to approximate such distributions
using low-dimensional ``distribution features''.
\Citet{Wang2008DualDP} introduce this idea in the context of
\emph{estimating} the occupancy distribution for a fixed policy rather
than directly solving the \emph{planning} problem of finding an
optimal policy ---
the authors suggest using their estimation procedure alongside
iterated policy improvement to find an optimal policy, but do not
characterize the convergence rate or approximation error of the
resulting algorithm.  More recently, \Citet{AYBaChMa19,AYBaMa14} propose
a stochastic gradient descent to be used on the Lagrangian derived
from the dual LP and derive a policy suboptimality bound for the
resulting poly-time algorithm; however, their results only apply under
some restrictive conditions.
%One critical assumption here is the feasibility of the dual ALP: 
%Without feasibility, the results of the paper are vacuous and unlike for the primal, feasibility of the dual is not trivial to guarantee. Other critical assumptions are that states have few predecessor state-action pairs, which can be access,
%the features are sparse and the MDP is uniformly fast mixing.

A major advantage of the~\cref{eq:dual-lp} is that its solutions
directly encode optimal policies (as discussed in
\cref{sec:lp-approach}) rather than just value functions.  When the
dual variables are approximated using ``distribution features'',
however, only a restricted class of policies can be represented.  For
example, when the distribution features are the occupancy measures of
a given set of ``base policies'', then solving the approximate dual LP
means finding the best mixture of the base policies.
\Citet{Banijamali2019} present an algorithm for this problem --- under
the additional assumption that the occupancy measures of the base
policies have large overlap.  They also show that, in general, this
problem is NP-hard to even approximate --- finding the best stochastic
policy in a restricted class can be \emph{harder} than finding an
optimal policy of the MDP\@.  We note in passing that the assumption
of a restricted class of policies that contains a close-to-optimal
policy can be considered complementary to our setting, where the
optimal value function is close to the span of a given feature
representation.

\paragraph*{Primal-Dual Methods}

There has been significant interest in applying recent advances in
primal-dual online optimization methods to planning in MDPs.  Since
the~\cref{eq:primal-lp} optimizes value functions while
the~\cref{eq:dual-lp} optimizes occupancy measures (i.e.\@ policies,
indirectly), primal-dual optimization can be seen as an ``actor-critic''
approach that finds both policies and value functions simultaneously.
\Citet{CogillPrimalDual2015} proposes solving the saddle-point form of
the LP in \cref{sec:lp-approach}, with no approximation and assuming
full knowledge of the transition matrix.
\Citet{ChenWangPrimalDual2016} adopt the same approach but with
stochastic updates using random samples of state transitions.
\Citet{chen2018scalable} extend this idea to large MDPs using
low-dimensional feature representations to approximate both the primal
and dual variables.  \Citet{bas2019faster} identify a ``coherence''
condition on the primal and dual feature representations that is
necessary to extract close-to-optimal policies from saddle-point
solutions with function approximation --- they also point out that,
without such an assumption, the policy suboptimality bound of
\citeauthor{chen2018scalable} can scale with the number of states of
the MDP (or worse).  We can avoid this issue in our setting because we
use the approximate solutions of~\cref{eq:corelp-saddle} only to
select one action, not an entire policy, unlike all these cited works.

\section{Conclusions}%
\label{sec:conc}

We presented an approach to efficient online planning in large-scale
$\gamma$-discounted MDPs in the presence of %
\begin{enumerate*}[label=\emph{(\roman*)},itemjoin={{; }},itemjoin*={{; and }}]
\item a (relatively) weak $d$-dimensional feature representation
\item a core set of $\nCoreStates$ states whose features' convex hull
  covers the features of other states
\item a stochastic simulator of the MDP
\end{enumerate*}\@. %
Our main contribution is an online planning algorithm that, for any
target precision $\epsilon$, achieves a value loss of
$O(\epsapprox/{(1-\gamma)}^2 + \epsilon)$, where
$\epsapprox$~\cref{eq:epsapprox} is the best achievable error in
uniformly approximating the optimal value function of the MDP using
the given feature representation.  When the MDP has $A$ actions per
state, the algorithm's runtime is
$\text{poly}(1/\epsilon,d,m,A,1/(1-\gamma))$, which is independent of
the number of states in the MDP\@.  Our work builds upon the
approximation error bound of \citet{Lakshminarayanan2018} for
relaxations of the approximate linear program.

% \Citet{du2019good} point out that it remains an open problem whether
% query-efficient planning is possible using a simulator and features
% that can well-approximate the optimal action-value function, with no
% additional conditions.  We resolve this open problem in the special
% case when a small set of core states is available such that the
% features at all states are in the convex hull of the features of the
% chosen states. It remains an interesting question as to whether this
% assumption can be removed without jeopardizing efficient planning.
\Citet{du2019good} point out that it remains an open problem whether
query-efficient planning is possible in large MDPs using only a
simulator and features that have a small approximation error $\epx$,
with no additional assumptions.  Our algorithm resolves this open
problem in the special case when a small set of core states is
available, i.e., when $m=\mathrm{poly}(d)$. 
It remains an intriguing question whether this assumption
can be removed without jeopardizing efficient planning.
Other interesting questions are whether the results can be extended to 
smoothed ALPs \citep{SALP}, and whether the adaptive constraint generation of
\citet{petrik} can be used to reduce the dependence on the planning horizon.
 
% Our algorithm positively answers the open question of \citet{du2019good},
% on the existence of a polynomial-time planner using a generative model
% of the MDP, which is unknown even when $\epsapprox$ (as defined above)
% is zero. However, our answer crucially depends on the availability of a
% small set of core states. Whether this condition can be removed
% without jeopardizing efficient planning remains an intriguing open
% question.

To achieve our results, we make several novel technical contributions:
We slightly change the ALP approach of \citet{Lakshminarayanan2018},
adding extra constraints and using a saddle-point formulation.  We
then show that near-optimal action distributions can be extracted from
approximate solutions of the saddle-point problem.  We solve the
saddle-point problem using a stochastic approximation algorithm,
Stochastic Mirror-Prox \citep{JuditskyStochMP2011} --- a first-order
primal-dual optimization method that uses stochastic gradient
estimates, which in our case are provided by the simulator. %
\todor{Check if this is true --- Unlike
previous saddle-point-based MDP algorithms, we avoid the need for a
known upper bound on the norm of the (unknown) feature weight vector
that gives the best approximation to the (unknown) optimal value
function.} %
We believe that these techniques and ideas can find applications in
other problems beyond our work.

%%%%%%%%%%%%%%%%%%%%%%%%%%%%%%%%%%%%%%%%%%%%%%%%%%%%%%%%%%%%%%%%%%%%%%%%

\ifpreprint\else
\section*{Broader Impact}
Our research has the nature of basic science --- we are working on
foundational improvements to reinforcement learning algorithms.
We are not targeting any specific applications, and it is hard to
foresee any societal consequences beyond those brought about by
advancing the state of our knowledge of machine learning.
\fi

\section*{Acknowledgements}

Csaba Szepesv\'ari gratefully acknowledges funding from the Canada
CIFAR AI Chairs Program, the Alberta Machine Intelligence Institute
(Amii), and the Natural Sciences and Engineering Research Council of
Canada (NSERC).

\bibliographystyle{customnat}
{
  \def\UrlFont{\rmfamily\small}
  \newcommand*{\urlprefix}{}
  \bibliography{references}
}

\clearpage{}
\appendix
\thispagestyle{plain}

\phantomsection{}%
\addcontentsline{toc}{chapter}{\AppendixName}%
\label{appendix}

{\centering\normalfont\LARGE\bfseries\AppendixName\par}
\vspace{1ex}

\section*{Index of Notation}%
\addcontentsline{toc}{section}{Index of Notation}%
\label{sec:notation-index}%

For the convenience of the reader, we have collected the most frequently used symbols and their meanings in the following table:
\begin{description}[style=sameline,leftmargin=5em,nosep]
\item[$\Real,\Real_+$] real numbers; non-negative real numbers.
\item[$\Real^d, \Real^{m\times n}$] $d$-dimensional vectors; matrices of
  size $m\times n$.
\item[$\vece_i, \vec0, \vec1$] standard basis vector: $e_{i,i} = 1$ and
  $e_{i,j} = 0$ for $i \ne j$; constant zero or one vectors.
\item[$\vec a \oplus \vec b$] concatenation of vectors: if $\vec a \in
  \Real^n$ and $\vec b\in\Real^m$, then $\vec a \oplus \vec b \in
  \Real^{n+m}$.
\item[$\States, \Actions, \CoreStates$] sets of states,
  actions, and core states (\cref{sec:background,ass:core-states}).
\item[$S, A, \nCoreStates$] number of states $\abs{\States}$, actions
  $\abs{\Actions}$, and core states $\abs{\CoreStates}$, respectively.
\item[$s_0, s, s', a$] planning state (\cref{sec:corelp}) and other
  states $\in \States$; actions $\in \Actions$.
\item[$\vP, \vE$] row-stochastic matrices in $\Real_+^{SA\times S}$
  (\cref{sec:background,sec:lp-approach}); $\vE_{sa} = \vece_{\!s} \in \Real^S$.
\item[$\vecr, \hat r$] expected rewards $\in\Real^{SA}$
  (\cref{sec:background}); random reward $\in [-1,1]$
  (\cref{ass:simulator}).
\item[$\gamma\in\lbrack 0, 1 \rparen$] discount factor (\cref{sec:background}).
\item[$v, v^*, v_\pi$] value functions $S\to\Real$.
\item[$\Delta_{\States},\Delta_{\Actions}$] sets of probability
  distributions over states and actions.
\item[$\vmu,\vpi, \vpi(s)$] probability distributions in
  $\Delta_{\States}$ and $\Delta_{\Actions}$, respectively; policy in
  $\States\to\Delta_{\Actions}$.
\item[$\vPhi, \vphi_s, \vphi_0$] feature matrix $\in\Real^{S\times d}$;
  state features $\in \Real^d$; features of planning state $\vphi_{s_0}$.
\item[$\epx$] approximation error of $\vPhi$:
  $\min_{\vtheta\in\Real^d}\norm{\infty}{\vecv^* - \vPhi\vtheta}$.
\item[$\vW, \vWstar$] constraint matrices in $\set{0,1}^{(1+m)A\times SA}$
  and $\set{0,1}^{mA\times SA}$ (\cref{thm:corelp,thm:lralp}).
\item[$\vlambda,\vlambda_*,\vtheta$] dual variables
  $\in \Real_+^{(1+m)A}$ and $\in \Real_+^{mA}$, respectively; primal variables $\in\Real^d$.
\item[$\Lambda, \Lambda_\gamma, \Bee$] dual spaces
  $\subset \Real_+^{(1+m)A}$; primal space $\subset\Real^d$
  (\cref{thm:corelp,lemma:saddlepoint-subopt,lemma:corestomp-subopt}).
% \item[$\vec A \otimes \vec B = \vec C$] Kronecker product: for
%   $\vec A\in\Real^{n\times m}$ and $\vec B\in\Real^{p\times q}$,
%   $\vec C \in \Real^{np\times mq}$ with $C_{ik,jl} = A_{i,j} B_{k,l}$.
\item[$\norm{*}{\wildcard}$] dual norm of $\norm{}{\wildcard}$:
  defined by $\norm{*}{\vec u} = \sup_{\norm{}{\vec x} = 1} \innerp{\vec u}{\vec x}$.
\end{description}

\section{Proofs}%
\label{sec:proofs}

\subsection{Approximation Error for the Linearly Relaxed Approximate LP}%
\label{sec:lralp-proofs}

\newcommand{\Jalp}{{J^*_{\mathrm{ALP}}}}
\newcommand{\Jlra}{{J^*_{\mathrm{LRA}}}}
\newcommand{\Vlralp}{{V_{\mathrm{LRALP}}}}

We start by recalling and improving the approximation error bounds for
the Linearly Relaxed Approximate Linear Program~\cref{eq:lralp} of
\citet{Lakshminarayanan2018}.
\begin{theorem}%
  \label{thm:lralp}
  Suppose \cref{ass:function-approx,ass:core-states} hold.  Define the
  matrix $\vWstar \in \set{0,1}^{mA\times SA}$ with rows
  ${[\vWstar]}_{sa} = \vece_{\!sa} \in \Real^{SA}$
  ($s\in\CoreStates,a\in\Actions$).
  For any (possibly unnormalized) initial distribution $\vmu\in\Real_+^{S}$,
  \begin{align}
    \Vlralp(\vmu)
    &\defeq \min\;\set{\,\transp\vmu\vPhi\vtheta \given \vtheta\in\Real^d,\,
      \vWstar\vecr + \vWstar(\gamma\vP - \vE)\vPhi\vtheta \le \vec0\,}
      \tag{LRALP$_{\vmu}$}\label{eq:lralp}.
  \end{align}
  The value of~\cref{eq:lralp} is close to the optimal value of that
  initial distribution:
  \begin{align*}
    \abs{\Vlralp(\vmu) - \transp\vmu\vecv^*} &\le \frac{10\norm{1}{\vmu}\epsapprox}{1-\gamma}.
  \end{align*}
\end{theorem}

This result follows from \citet[Theorem~IV.1]{Lakshminarayanan2018},
which we will not reproduce here for brevity.  The error bound there is
$2\norm{1}{\vmu}(3\epsapprox + \norm{\infty}{\Jalp -
  \Jlra})/(1-\gamma)$, defining
\begin{align*}\SwapAboveDisplaySkip{}
  \Jalp(s)
  &\defeq \min\;\set{\,\transp\vphi_s\vtheta \given \vtheta\in\Real^d,\, \vPhi\vtheta\ge\vecv^*\,}, \\
  \Jlra(s)
  &\defeq \min\;\set{\,\transp\vphi_s\vtheta \given \vtheta\in\Real^d,\, \vWstar\vE\vPhi\vtheta\ge\vWstar\vE\vecv^*\,}. 
\end{align*}
It only remains for us to bound $\norm{\infty}{\Jalp - \Jlra}$,
improving upon Theorem~IV.2~\citep{Lakshminarayanan2018}:
\begin{lemma}
  Under the conditions of \cref{thm:lralp}, $\norm{\infty}{\Jalp - \Jlra} \le 2\epsapprox$.
  \begin{proof}
    By \cref{ass:function-approx}, the optimal value function is
    well-approximated by the feature representation;
    $\vecv^* = \vPhi\vtheta+\vdelta$ for some $\vtheta\in\Real^d$ and
    $\vdelta\in\Real^S$ with
    $\norm{\infty}{\vdelta} \le \epsapprox$.  By
    \cref{ass:core-states}, $\vPhi = \vZ\vPhiCore$, so
    $\vecv^* = \vZ\vPhiCore\vtheta + \vdelta$.  We use these facts
    after writing the linear program defining $\Jalp(s)$ in its dual
    form:
    \begin{align*}
      \Jalp(s)
      &= \max\;\set{\,\transp\vmu\vecv^* \given \vmu\in\Real_+^S,\, \transp\vmu\vPhi = \transp\vphi_s\,} \\
      &= \max\;\set{\,\transp\vmu(\vZ\vPhiCore\vtheta + \vdelta) \given \vmu\in\Real_+^S,\,
        \transp\vmu\vZ\vPhiCore = \transp\vphi_s\,}
        \intertext{By \cref{ass:function-approx}, there is some
        $\veta\in\Real^d$ such that $\vPhi\veta = \vec1$.  If $\transp\vmu\vPhi = \transp\vphi_s$, then
        $\norm{1}{\vmu} = \transp\vmu\vec1 = \transp\vmu\vPhi\veta =
        \transp\vphi_s\veta = 1$, which means that
        $\transp\vmu\vdelta \le \norm{\infty}{\vdelta}$. Replacing
        $\transp\vmu\vdelta$ with $\norm{\infty}{\vdelta}$ in the
        objective increases its value; we move the resulting constant term out of the
        maximization:}
      &\le \norm{\infty}{\vdelta} + \max\;\set{\,\transp\vmu\vZ\vPhiCore\vtheta \given \vmu\in\Real_+^S,\,
        \transp\vmu\vZ\vPhiCore = \transp\vphi_s\,}
        \intertext{The objective and constraints of this maximization
        problem depend on $\vmu$ only through $\transp\vmu\vZ$.  Thus we can
        replace $\transp\vmu\vZ$ with $\vmu_*\in\Real_+^m$, which can only expand the
        feasible set of the maximization and increase its value:}
      &\le \norm{\infty}{\vdelta} + \max\;\set{\,\transp\vmu_*\vPhiCore\vtheta \given \vmu_*\in\Real_+^m,\,
        \transp\vmu_*\vPhiCore = \transp\vphi_s\,}
        \intertext{The matrix $\vec U \in \set{0,1}^{m\times S}$ with
        rows ${[\vec U_s]}_{s\in\CoreStates} = \vece_s$ can be used to ``select'' the core
        state features from $\vPhi$, giving $\vPhiCore = \vec U\vPhi$:}
      &= \norm{\infty}{\vdelta} + \max\;\set{\,\transp\vmu_*\vec U\vPhi\vtheta \given \vmu_*\in\Real_+^m,\,
        \transp\vmu_*\vec U\vPhi = \transp\vphi_s\,}
        \intertext{By a similar argument as before, we see that
        $\norm{1}{\transp\vmu_*\vec U} = 1$.  We
        add $\transp\vmu_*\vec U\vdelta + \norm{\infty}{\vdelta} \ge 0$
        to the objective (increasing its value), then move the constant out:}
      &\le 2\norm{\infty}{\vdelta} + \max\;\set{\,\transp\vmu_*\vec U\vPhi\vtheta + \transp\vmu_*\vec U\vdelta
        \given \vmu_*\in\Real_+^m,\, \transp\vmu_*\vec U\vPhi = \transp\vphi_s\,} \\
      &= 2\norm{\infty}{\vdelta} + \max\;\set{\,\transp\vmu_*\vec U\vecv^*
        \given \vmu_*\in\Real_+^m,\, \transp\vmu_*\vec U\vPhi = \transp\vphi_s\,} \\
      &= 2\norm{\infty}{\vdelta} + \min\;\set{\,\transp\vphi_s\vtheta
        \given \vtheta\in\Real^d,\, \vec U\vPhi\vtheta \ge \vec U\vecv^*\,},
    \end{align*}
    where the last step is obtained by writing the dual of the linear
    program in the previous step.  Now observe that the constraint
    $\vec U\vPhi\vtheta \ge \vec U\vecv^*$ is equivalent to the
    constraint $\vWstar\vE\vPhi\vtheta \ge \vWstar\vE\vecv^*$ in the
    definition of $\Jlra$ --- both of them require that
    $\vphi_s\vtheta \ge v^*_s$ for $s\in\CoreStates$.  Thus we have
    shown that $\Jalp(s) - \Jlra(s) \le 2\epsapprox$ for all
    $s\in\States$.  We also know that $\Jalp(s) \ge \Jlra(s)$, since
    $\Jlra(s)$ is a relaxation of $\Jalp(s)$.  It follows that
    $\norm{\infty}{\Jalp - \Jlra} \le 2\epsapprox$.
  \end{proof}
\end{lemma}

\subsection{Proof of \texorpdfstring{\cref{thm:corelp}}{Theorem~2} ---
  Approximation Error for CoreLP}%
\label{sec:corelp-proof}

\TheoremCoreLP*

By the definition of $\Lambda \subset \Real_+^{(1+m)A}$, we can
decompose its elements as $\vlambda = \vpi\oplus\vlambda_*$, with
$\vpi\in\Delta_{\Actions}$ as in the statement of the
\lcnamecref{thm:corelp} and $\vlambda_*\in\Real_+^{mA}$ defined by
$\lambda_{*,sa} = \lambda_{sa}$ for $s\in\CoreStates, a\in\Actions$
--- in other words,
$\Lambda \cong \Delta_{\Actions} \times \Real_+^{mA}$.  The main idea
of the proof is that when $\vlambda$ is a solution
of~\cref{eq:corelp}, then $\vlambda_*$ is a
solution for the dual form of~\cref{eq:lralp} from \cref{thm:lralp}.
To make this connection between the two problems more precise, let us
write the saddle-point forms of~\cref{eq:lralp,eq:corelp}:
\begin{alignat}{4}
  \Vlralp(\vmu)
  &= \smashoperator[l]{\max_{\vlambda_*\in\Real_+^{mA}}}\, &\min_{\vtheta\in\Real^d \vphantom{\Real_+^{mA}}}
    \big[&&\,g_{\vmu}(\vlambda_*, \vtheta)
    &\defeq \transp\vlambda_*\vWstar\vecr + \transp\vmu\vPhi\vtheta
    + \transp\vlambda_*\vWstar(\gamma\vP - \vE)\vPhi\vtheta\,&&\big]
    \tag{Saddle LRALP$_{\vmu}$}\label{eq:lralp-saddle} \\
  V^\dagger
  &= \max_{\vlambda\in\Lambda \vphantom{\cramped{\Real^d}}} &\min_{\vtheta\in\Real^d}
    \big[&&\,f(\vlambda, \vtheta)
    &\defeq \transp\vlambda\vW\vecr + \transp\vece_{\!s_0}\vPhi\vtheta
    + \transp\vlambda\vW(\gamma\vP - \vE)\vPhi\vtheta\,&&\big]
    \tag{Saddle CoreLP}\label{eq:corelp-saddle-again}
\end{alignat}

\begin{lemma}[Corresponding \cref{eq:lralp,eq:corelp} solutions]%
  \label{lemma:corelp-lralp}
  Let $\vlambda \in \Lambda \subset \Real_+^{(1+m)A}$ be arbitrary and
  decompose it as $\vlambda = \vpi \oplus \vlambda_*$, where
  $\vpi\in\Delta_{\Actions}$ and $\vlambda_*\in\Real_+^{mA}$.  Define
  the distribution $\vmu_{\vpi} \in \Delta_{\States}$ as
  \begin{align*}
    \transp\vmu_{\vpi} &\defeq \sum_{a\in\Actions} \pi(a) \, \vP_{s_0a},
    &\text{where $\pi(a) \defeq \lambda_{s_0a}$ for $a\in\Actions$.} \\
    \intertext{Then, for any $\vtheta\in\Real^d$, and $g_{\vmu}(\vlambda_*, \vtheta)$ and
    $f(\vlambda, \vtheta)$ as
    in~\cref{eq:lralp-saddle,eq:corelp-saddle-again},}
    f(\vlambda, \vtheta) &= \sum_{a\in\Actions} \pi(a) \, r(s_0, a) + g_{\gamma\vmu_{\vpi}}(\vlambda_*, \vtheta),
    &\text{where $\lambda_{*,sa} = \lambda_{sa}$ for $s\in\CoreStates, a\in\Actions$.}
  \end{align*}
\end{lemma}
$\vP_{s_0a}$ is the next-state distribution for action $a$ at state
$s_0$ --- thus the distribution $\vmu_{\vpi} \in \Delta_{\States}$
defined here is the expected next-state distribution when an action
$a\sim\vpi$ is taken at state $s_0$.  This
\lcnamecref{lemma:corelp-lralp} therefore connects solutions
of~\cref{eq:corelp} with~\cref{eq:lralp} when $\vmu=\gamma\vmu_{\vpi}$
is the discounted next-state distribution for action $a\sim\vpi$.

\begin{proof}[Proof of \cref{lemma:corelp-lralp}]
  Recall that $\vW$ and $\vWstar$ (defined in
  \cref{thm:corelp,thm:lralp}) are related --- the rows of $\vWstar$
  correspond to state-action pairs in $\CoreStates\times\Actions$, to
  which $\vW$ adds $A$ more rows corresponding to the actions at the
  current planning state $s_0$.  Thus
  \begin{align}
    \transp\vlambda\vW
    &= \sum_{a\in\Actions} \pi_a \transp\vece_{\!s_0a} + \transp\vlambda_*\vWstar, \label{eq:corelp-lralp-W}\\
    \intertext{which upon multiplying by $\vecr$ gives}
    \transp\vlambda\vW\vecr
    &= \sum_{a\in\Actions} \pi_a r_{s_0a} + \transp\vlambda_*\vWstar\vecr. \label{eq:corelp-lralp-reward}\\
    \intertext{Using~\cref{eq:corelp-lralp-W} again,}
    \transp\vece_{\!s_0} + \transp\vlambda\vW(\gamma\vP - \vE)
    &= \transp\vece_{\!s_0}
      + \sum_{a\in\Actions}\pi_a(\gamma\vP_{s_0a} - \vE_{s_0a})
      + \transp\vlambda_*\vWstar(\gamma\vP - \vE), \notag\\
    &= \brck[\big]{\transp\vece_{\!s_0} - \sum_{a\in\Actions} \pi_a \vE_{s_0a}}
      + \gamma\brck[\big]{\!\sum_{a\in\Actions}\pi_a\vP_{s_0a}}
      + \transp\vlambda_*\vWstar(\gamma\vP - \vE). \notag\\
    \intertext{The first term is zero because $\vE_{s_0a} = \vece_{\!s_0}$ for all
    $a\in\Actions$, and the second term becomes $\gamma\vmu_{\vpi}$
    when we substitute the definition of $\vmu_{\vpi}$.
    We then multiply both sides by $\vPhi\vtheta$:}
    \transp\vece_{\!s_0}\vPhi\vtheta + \transp\vlambda\vW(\gamma\vP - \vE)\vPhi\vtheta
    &= \gamma\transp\vmu_{\vpi}\vPhi\vtheta + \transp\vlambda_*\vWstar(\gamma\vP - \vE)\vPhi\vtheta.
      \notag\\
    \intertext{Adding this to~\cref{eq:corelp-lralp-reward} gives}
    \transp\vlambda\vW\vecr + \transp\vece_{\!s_0}\vPhi\vtheta +
    \transp\vlambda\vW(\gamma\vP - \vE)\vPhi\vtheta
    &= \sum_{a\in\Actions} \pi_a r_{s_0a} + \transp\vlambda_*\vWstar\vecr +
      \gamma\transp\vmu_{\vpi}\vPhi\vtheta + \transp\vlambda_*\vWstar(\gamma\vP - \vE)\vPhi\vtheta, \notag\\
    \shortintertext{where we substitute the definitions of $f(\vlambda, \vtheta)$ and
    $g_{\gamma\vmu_{\pi}}(\vlambda_*, \vtheta)$ to get the desired result:}
    f(\vlambda, \vtheta) &= \sum_{a\in\Actions} \pi_a r_{s_0a} + g_{\gamma\vmu_{\vpi}}(\vlambda_*, \vtheta). \notag\qedhere
  \end{align}
\end{proof}

\begin{proof}[Proof of \cref{thm:corelp}]
  Using the decomposition
  $\Lambda \cong \Delta_{\Actions} \times \Real_+^{mA}$ in  \cref{eq:corelp-saddle-again}:
  \begin{align*}
    V^\dagger
    &= \max_{\vpi\in\Delta_{\Actions}\vphantom{\Real^{A}}} \adjustlimits \max_{\vlambda_*\in\Real_+^{mA}} \min_{\vtheta\in\Real^d}
      f(\vpi\oplus\vlambda_*, \vtheta)
    &\text{(where $\vpi \oplus \vlambda_* = \vlambda \in\Lambda$)}\\
    &= \max_{\vpi\in\Delta_{\Actions}\vphantom{\Real^{A}}} \adjustlimits \max_{\vlambda_*\in\Real_+^{mA}} \min_{\vtheta\in\Real^d} \;
      \brck[\big]{\sum_{a\in\Actions} \pi_a r_{s_0a} + g_{\gamma\vmu_{\vpi}}(\vlambda_*, \vtheta)}
    &\text{(using \cref{lemma:corelp-lralp})}\\
    &= \max_{\vpi\in\Delta_{\Actions}} \;\brck[\big]{
      \sum_{a\in\Actions} \pi_a r_{s_0a} + \adjustlimits \max_{\vlambda_*\in\Real_+^{mA}} \min_{\vtheta\in\Real^d}
      g_{\gamma\vmu_{\vpi}}(\vlambda_*, \vtheta)
      } \\
    &= \max_{\vpi\in\Delta_{\Actions}} \;\brck[\big]{\, q^\dagger(\vpi) \, \defeq
      \sum_{a\in\Actions} \pi_a r_{s_0a} + \Vlralp(\gamma\vmu_{\vpi}) \,}.
      &\text{(from \cref{eq:lralp-saddle})}
  \end{align*}
  We now turn our attention to bounding $V^\dagger$.  From
  \cref{thm:lralp}, we know that
  $\abs{\Vlralp(\gamma\vmu_{\vpi}) - \gamma\transp\vmu_{\vpi}\vecv^*} \le
  10\gamma\epsapprox/(1-\gamma)$ for any distribution over states
  $\vmu_{\vpi}\in\Delta_{\States}$.  Through a slight abuse of
  notation, we define
  $q^*(s_0, \vpi) \defeq \sum_a \pi_a \, r_{s_0a} +
  \gamma\transp\vmu_{\vpi} \vecv^*$ as a generalization of the
  standard $q^*(s,a)$ value function to action distributions.  Note
  that we will only need $q^*(s_0,\wildcard)$, for which this abuse is
  `sensible'. Then for all $\vpi\in\Delta_{\Actions}$,
  \begin{align}
    \abs{q^\dagger(\vpi) - q^*(s_0, \vpi)}
    &= \abs{\Vlralp(\gamma\vmu_{\vpi}) - \gamma\transp\vmu_{\vpi}\vecv^*} \le \frac{10\gamma\epsapprox}{1-\gamma}.
    \label{eq:qdiff}
  \end{align}
  We also know that $v^*(s_0) = \max_{\vpi \in \Delta_{\Actions}} q^*(s_0,\vpi)$ (the equality happens with
  $\vpi^* = \vece_{a^*}$ for an optimal action $a^*$). 
  Hence,
  \begin{align*}
	\abs{V^\dagger - v^*(s_0)}
	&=
	\abs{\max_{\vpi \in \Delta_{\Actions}} q^\dagger(\vpi) - \max_{\vpi\in \Delta_{\Actions}} q^*(s_0,\vpi)} 
	 \le
	\max_{\vpi \in \Delta_{\Actions}} \abs{ q^\dagger(\vpi) - q^*(s_0,\vpi)} 
	\le 
	\frac{10\gamma\epx}{1-\gamma}\,,
  \end{align*}
  where the last inequality follows from~\cref{eq:qdiff}.
  
  For the second part of the result, let $\vlambda^\dagger$ be a
  maximizer of~\cref{eq:corelp} and $\vpi^\dagger$ be the
  action-distribution component (as before) so that
  $V^\dagger = q^\dagger(\vpi^\dagger)$. Then, using
  again~\cref{eq:qdiff}, combined with the last inequality,
  \begin{align*}
    \sum_{a\in\Actions} \pi^\dagger(a) \, q^*(s_0, a) \equiv q^*(s_0, \vpi^\dagger)
    &\ge q^\dagger(\vpi^\dagger) - \frac{10\gamma\epsapprox}{1-\gamma} \\
    &= V^\dagger - \frac{10\gamma\epsapprox}{1-\gamma} \\
    &\ge v^*(s_0) - \frac{20\gamma\epsapprox}{1-\gamma}.
  \end{align*}
  Reordering gives the desired result, namely that
  $v^*(s_0) - \sum_a \pi^\dagger(a) \, q^*(s_0, a) \le 20\gamma\epsapprox/(1-\gamma)$.
\end{proof}

%\if0
\subsection{Proof of \texorpdfstring{\cref{thm:corestomp}}{Theorem~3}
  --- Error Bounds for the CoreStoMP Algorithm}%
\label{sec:corestomp-proof}

\TheoremCoreStoMP*

The proof of this theorem has two main ingredients: First, in
\cref{lemma:saddlepoint-subopt}, we show that approximate solutions
of~\cref{eq:corelp-saddle} can be used to recover near-optimal action
distributions for the planning state $s_0$ --- the approximation
quality is measured by the \emph{duality gap}.  Second, in
\cref{lemma:corestomp-subopt}, we bound the expected duality gap of
the Stochastic Mirror-Prox algorithm when specialized to our setting.

\begin{lemma}[Approximate~\cref{eq:corelp-saddle} solutions]\label{lemma:saddlepoint-subopt}
  Suppose $\Bee\subset\Real^d$ and $C_\Bee \ge 0$ are chosen such that,
  for any distribution over states $\vmu\in\Delta_{\States}$, there is
  some $\vtheta \in \Bee$ that is feasible for~\cref{eq:lralp} and at
  most $C_\Bee$-suboptimal.  Define
  \begin{align}
    \Lambda_\gamma
    &\defeq \set{\vlambda\in\Lambda \given
      \norm{1}{\vlambda} = 1/(1-\gamma)}, \label{eq:dual-domain}
  \end{align}
  a subset of the set $\Lambda\subset\Real_+^{(1+m)A}$ from
  \cref{thm:corelp}.  Define the \emph{$\Bee$-bounded duality gap} of
  an approximate solution of~\cref{eq:corelp-saddle} as
  \begin{align}
    \delta_{\Bee}(\hat\vlambda, \hat\vtheta)
    &\defeq \max_{\vlambda\in\Lambda_\gamma} f(\vlambda,\hat\vtheta) -
      \inf_{\vtheta\in\Bee} f(\hat\vlambda,\vtheta),
    &\text{where $\hat\vlambda\in\Lambda$ and $\hat\vtheta\in\Real^d$.}
      \label{eq:duality-gap}
  \end{align}
  For any $\hat\vlambda\in\Lambda$ and $\hat\vtheta\in\Real^d$, let
  $\hat\vpi$ be the action distribution component of $\hat\vlambda$,
  as in \cref{thm:corelp}.  Then
  \begin{align*}
    v^*(s_0) - \sum_{a\in\Actions} \hat\pi(a) \, q^*(s_0, a)
    &\le \frac{20\gamma\epx}{1-\gamma} + \gamma C_{\Bee} + \delta_{\Bee}(\hat\vlambda,\hat\vtheta).
  \end{align*}
\end{lemma}

This \lcnamecref{lemma:saddlepoint-subopt} generalizes the second
result of \cref{thm:corelp} in two ways: First, the Stochastic
Mirror-Prox algorithm does not produce exact solutions
of~\cref{eq:corelp-saddle}; the optimization error is measured by the
duality gap --- here we see the effect of a non-zero duality gap on
the resulting action distribution. Second, the primal variables
$\vtheta$ in~\cref{eq:corelp-saddle} have the unbounded domain
$\Real^d$, whereas the Stochastic Mirror-Prox algorithm requires the
optimization domain to have a bounded \emph{diameter}; see
\nameref{sec:stochmp-proof}.  This
\lcnamecref{lemma:saddlepoint-subopt} shows that restricting $\vtheta$
to a large-enough bounded set $\Bee$ only incurs an additional
$C_{\Bee}$ error.  Indeed, the second issue is related to the
first --- an \emph{unbounded} form of the duality gap would be
infinite for \emph{any} approximate solution, making it useless as a
measure of optimization accuracy; the $\Bee$-bounded duality gap
therefore addresses both these issues:

\begin{claim}\label{claim:corelp-saddle-subopt}
  For any $\hat\vlambda \in \Lambda$, $\hat\vtheta \in \Real^d$,
  $\vtheta \in \mathcal{B} \subset \Real^d$, and
  $\delta_{\Bee}(\hat\vlambda,\hat\vtheta)$ being the $\Bee$-bounded
  duality gap~\cref{eq:duality-gap},
  \begin{align*}
    V^\dagger
    &\le f(\hat\vlambda, \vtheta) + \delta_{\Bee}(\hat\vlambda, \hat\vtheta).
  \end{align*}

  \begin{proof}
    Let $\vlambda^* \in \Lambda \subset \Real_+^{(1+m)A}$ be a maximizer of
    \cref{eq:corelp} --- this exists because the optimization is
    bounded (\cref{thm:corelp}).  Then
    \begin{align}
      \vec0
      &=\transp\vphi_0 + \transp{\vlambda^*}\vW(\gamma\vP - \vE)\vPhi
      &\text{(since $\vlambda^*$ is feasible for \cref{eq:corelp})}
        \label{eq:corelp-constraint}\\
      &=\transp\vphi_0\hat\vtheta + \transp{\vlambda^*}\vW(\gamma\vP - \vE)\vPhi\hat\vtheta.
      &\text{(multiplying by $\hat\vtheta\in\Real^d$)} \notag\\
      \intertext{Since, $\vlambda^*$ is a maximizer of \cref{eq:corelp},
      $V^\dagger = \transp{\vlambda^*}\vW\vecr$:}
      V^\dagger
      &=\transp{\vlambda^*}\vW\vecr + \transp\vphi_0\hat\vtheta +
        \transp{\vlambda^*}\vW(\gamma\vP - \vE)\vPhi\hat\vtheta
      &\text{(adding $V^\dagger$ on l.h.s.\@ and
        $\transp{\vlambda^*}\vW\vecr$ on r.h.s.)} \notag\\
      &= f(\vlambda^*, \hat\vtheta).
      &\text{(definition of $f$ from \cref{eq:corelp-saddle-again})}
        \label{eq:corelp-saddle-opt}
    \end{align}
    \Cref{ass:function-approx} tells us that $\vPhi\veta = \vec1$ for
    some $\veta\in\Real^d$ --- multiplying~\cref{eq:corelp-constraint}
    by $\veta$, we see that $\vlambda^*$ must satisfy
    $1 + \gamma\norm{1}{\vlambda^*} = \norm{1}{\vlambda^*}$, as does
    any other feasible solution of \cref{eq:corelp}.  In particular,
    this means that $\norm{1}{\vlambda^*} = 1/(1-\gamma)$ and so
    $\vlambda^*\in\Lambda_\gamma$.  Using the definition of
    $\delta_{\Bee}$ from~\cref{eq:duality-gap},
    \begin{align*}
      \delta(\hat\vlambda, \hat\vtheta)
      &\ge f(\vlambda^*, \hat\vtheta) - f(\hat\vlambda, \vtheta)
      &\text{(since $\vlambda^*\in\Lambda_\gamma$ and $\vtheta\in\Bee$)} \\
      &= V^\dagger - f(\hat\vlambda, \vtheta).
      &\text{(using \cref{eq:corelp-saddle-opt})} &\qedhere
    \end{align*}
  \end{proof}
\end{claim}

\begin{claim}\label{claim:lralp-saddle-subopt}
  For any $\hat\vlambda_* \in \Real_+^{mA}$ and distribution over states
  $\vmu \in \Delta_{\States}$, suppose $\vtheta\in\Real^d$ is feasible for
  $\cref{eq:lralp}$ and at most $C_{\Bee}$-suboptimal.  Then, with $g_{\vmu}$
  being the objective function of \cref{eq:lralp-saddle},
  \begin{align*}
    \Vlralp(\gamma\vmu) &\ge g_{\gamma\vmu}(\hat\vlambda_*, \vtheta) - \gamma C_{\Bee}.
  \end{align*}

  \begin{proof}
    Since $\vtheta$ is feasible for \cref{eq:lralp}, $\vWstar\vecr +
    \vWstar(\gamma\vP - \vE)\vPhi\vtheta \le \vec0$.  Multiplying both
    sides of this inequality by $\transp{\hat\vlambda}_* \ge \vec0$,
    \begin{align*}
      \transp{\hat\vlambda}_*\vWstar\vecr +
      \transp{\hat\vlambda}_*\vWstar(\gamma\vP - \vE)\vPhi\vtheta
      &\le 0 \\
      \transp{\hat\vlambda}_*\vWstar\vecr + \gamma\transp\vmu\vPhi\vtheta +
      \transp{\hat\vlambda}_*\vWstar(\gamma\vP - \vE)\vPhi\vtheta
      &\le \gamma\transp\vmu\vPhi\vtheta
      &\text{(adding $\gamma\transp\vmu\vPhi\vtheta$ to both sides)} \\
      g_{\gamma\vmu}(\hat\vlambda_*, \vtheta)
      &\le \gamma\transp\vmu\vPhi\vtheta.
        &\text{(definition of $g$ from \cref{eq:lralp-saddle})}
    \end{align*}
    Note that the choice of $\vmu$ does not affect the constraints
    of~\cref{eq:lralp}, only its objective function --- thus $\vtheta$
    is feasible for the problem defining $\Vlralp(\gamma\vmu)$ and is
    $\gamma C_{\Bee}$ suboptimal:
    $\gamma\transp\vmu\vPhi\vtheta \le \Vlralp(\gamma\vmu) + \gamma
    C_{\Bee}$.  Substituting this into the last inequality and
    rearranging gives the desired result.
  \end{proof}
\end{claim}

\begin{proof}[Proof of \cref{lemma:saddlepoint-subopt}]
  As in \cref{lemma:corelp-lralp}, we write
  $\hat\vlambda = \hat\vpi \oplus \hat\vlambda_*$ with
  $\hat\vpi\in\Delta_{\Actions}$ and $\hat\vlambda_*\in\Real_+^{mA}$
  and define $\vmu_{\hat\vpi} = \sum_a\hat\pi_a\vP_{s_0a}$.  By our
  assumption, there is some $\vtheta \in \Bee$ that is feasible and at
  most $C_{\Bee}$-suboptimal for \cref{eq:lralp} with
  $\vmu = \vmu_{\hat\vpi}$ --- this allows us to apply
  \cref{claim:corelp-saddle-subopt,claim:lralp-saddle-subopt} below:
  \begin{align*}
    v^*(s_0)
    &\le V^\dagger + \frac{10\gamma\epx}{1-\gamma}
    &\text{(\cref{thm:corelp})}\\
    &\le f(\hat\vlambda, \vtheta)
      + \delta_{\Bee}(\hat\vlambda, \hat\vtheta) + \frac{10\gamma\epx}{1-\gamma}
    &\text{(\Cref{claim:corelp-saddle-subopt})} \\
    &= \sum_{a\in\Actions} \hat\pi_a r_{s_0a} + g_{\gamma\vmu_{\hat\vpi}}(\hat\vlambda_*, \vtheta)
      + \delta_{\Bee}(\hat\vlambda, \hat\vtheta) + \frac{10\gamma\epx}{1-\gamma}
    &\text{(\Cref{lemma:corelp-lralp})} \\
    &\le \sum_{a\in\Actions} \hat\pi_a r_{s_0a} + \Vlralp(\gamma\vmu_{\hat\vpi})
      + \gamma C_{\Bee} + \delta_{\Bee}(\hat\vlambda, \hat\vtheta) + \frac{10\gamma\epx}{1-\gamma}
    &\text{(\Cref{claim:lralp-saddle-subopt})} \\
    &\le \sum_{a\in\Actions} \hat\pi_a r_{s_0a} + \gamma\transp\vmu_{\hat\vpi}\vecv^*
      + \gamma C_{\Bee} + \delta_{\Bee}(\hat\vlambda, \hat\vtheta) + \frac{20\gamma\epx}{1-\gamma}
    &\text{(\Cref{thm:lralp})} \\
    &= \sum_{a\in\Actions} \hat\pi_a q^*(s_0, a)
      + \gamma C_{\Bee} + \delta_{\Bee}(\hat\vlambda, \hat\vtheta) + \frac{20\gamma\epx}{1-\gamma}, \\
    \intertext{where the last step used $\sum_a\hat\pi_a r_{s_0a} +
    \gamma\transp\vmu_{\hat\vpi}\vecv^* = \sum_a\hat\pi_a(r_{s_0a} +
    \gamma\vP_{s_0a}\vecv^*) = \sum_a\hat\pi_a q^*(s_0,a)$.  Rearranging the
    inequality completes the proof:}
    v^*(s_0) &- \sum_{a\in\Actions} \hat\pi_a q^*(s_0, a)
    \le \frac{20\gamma\epx}{1-\gamma} + \gamma C_{\Bee} + \delta_{\Bee}(\hat\vlambda, \hat\vtheta). &&\qedhere
  \end{align*}
\end{proof}

The following \lcnamecref{lemma:corestomp-subopt} bounds the expected
duality gap of the Stochastic Mirror-Prox algorithm when applied to
our setting.  We defer the proof to \cref{sec:stochmp-proof}.

\begin{restatable}[Stochastic Mirror-Prox]{lemma}{TheoremStochMP}\label{lemma:corestomp-subopt}
  Using the constants $B$ and $C$ from \cref{thm:corestomp}, define
  \begin{align}
    \Bee
    &\defeq \set{\vtheta\in\Real^d \given
      \norm{2}{\vPhiCore\vtheta} \le B},
      \label{eq:primal-domain}
  \end{align}
  and let $\delta_{\Bee}(\hat\vlambda, \hat\vtheta)$ be the
  $\Bee$-bounded duality gap~\cref{eq:duality-gap}.  Then the results
  of running \cref{alg:corestomp} for $T$ iterations satisfy
  \begin{align*}\SwapAboveDisplaySkip{}
    \epsopt &\defeq \Ex{\delta(\hat\vlambda, \hat\vtheta)} \le \frac{14C}{\sqrt{3T}}.
  \end{align*}
\end{restatable}

\begin{proof}[Proof of \cref{thm:corestomp}]
  First, observe that $v^*(s_0) - q^*(s_0, a) \le 2/(1-\gamma)$ for any
  action $a$, since all the rewards lie in $[-1, 1]$ by
  \cref{ass:simulator}.  Thus, if $\epx > 1/16$ then
  $32\epx/(1-\gamma) > 2/(1-\gamma)$ and the result is trivially
  true.  From now on, we will assume that $\epx \le 1/16$.

  To prove our result, we will combine
  \cref{lemma:saddlepoint-subopt,lemma:corestomp-subopt}, for which we
  need to show that the set $\Bee$ defined in~\cref{eq:primal-domain}
  satisfies the requirements of \cref{lemma:saddlepoint-subopt}.
  Specifically, we need to show that $\Bee$ contains a feasible
  solution of the linear program~\cref{eq:lralp} with sub-optimality
  bounded by a constant $C_{\Bee}$.  Note that the constraints
  of~\cref{eq:lralp} do not depend on $\vmu$, only the objective
  function, so the choice of $\vmu$ does not affect feasibility.

  Since~\cref{eq:lralp} is a relaxation of the ALP (see
  \cref{sec:lp-approach}), any feasible solution of the ALP is also
  feasible for~\cref{eq:lralp}.  \Citet[Theorem~2]{ALP} show that the
  ALP has a feasible solution $\vPhi\vtheta$ that is close to
  $\vecv^*$ --- more precisely
  \begin{align}\SwapAboveDisplaySkip{}
    \norm{\infty}{\vPhi\vtheta  - \vecv^*} \le \frac{2\epx}{1-\gamma}.
    \label{eq:alpbasic}
  \end{align}
  Since $\norm{\infty}{\vecv^*} \le 1/(1-\gamma)$, we must have
  $\norm{\infty}{\vPhi\vtheta} \le (1+2\epsapprox)/(1-\gamma)$.  It
  follows that
  \begin{align*}
    \norm{2}{\vPhiCore\vtheta}
    &\le \sqrt{m}\norm{\infty}{\vPhiCore\vtheta}
      = \sqrt{m}\norm{\infty}{\vPhi\vtheta}
      \le \frac{(1+2\epx)\sqrt{m}}{1-\gamma}
      \le \frac{(9/8)\sqrt{m}}{1-\gamma} = B,
  \end{align*}
  where the first inequality is a property of the 2-norm, the next
  equality is thanks to \cref{ass:core-states}, and the last
  inequality is because we assumed $\epx \le 1/16$ --- this shows that
  $\vtheta \in \Bee$.  We also have
  \begin{align*}
    \transp\vmu\vPhi\vtheta - \transp\vmu\vecv^*
    &\le \norm{1}{\vmu} \norm{\infty}{\vPhi\vtheta - \vecv^*} \le
      \frac{2\epx}{1-\gamma}
    &\text{(using~\cref{eq:alpbasic} and $\norm{1}{\vmu} = 1$)} \\
    \abs{\transp\vmu\vecv^* - \Vlralp(\vmu)}
    &\le \frac{10\epx}{1-\gamma}
    &\text{(using \cref{thm:lralp} and $\norm{1}{\vmu} = 1$)} \\
    \shortintertext{We get the value of $C_{\Bee}$ by putting these two bounds together:}
    \transp\vmu\vPhi\vtheta - \Vlralp(\vmu)
    &\le \frac{12\epx}{1-\gamma} \eqdef C_{\Bee}.
  \end{align*}
  \Cref{lemma:saddlepoint-subopt} applied to $\Bee$ with this value of $C_{\Bee}$ gives
  \begin{align*}
    v^*(s_0) - \sum_{a\in\Actions} \hat\pi(a) \, q^*(s_0, a)
    &\le \frac{32\gamma\epx}{1-\gamma} +
      \delta_{\Bee}(\hat\vlambda,\hat\vtheta). \\
    \shortintertext{Taking expectations on both sides and
    substituting the value of $\epsopt = \Ex{\delta_{\Bee}(\hat\vlambda, \hat\vtheta)}$ from
    \cref{lemma:corestomp-subopt},}
    v^*(s_0) - \Ex{q^*(s_0, a)}
    &\le \frac{32\gamma\epx}{1-\gamma} + \frac{14C}{\sqrt{3T}}.
  \end{align*}
  We drop the $\gamma$ factor from the leading term and plug in the
  value of $C$ from the statement of \cref{thm:corestomp} to finish
  the proof.
\end{proof}

\subsection{Proof of
  \texorpdfstring{\Cref{lemma:corestomp-subopt}}{Lemma~11} ---
  Stochastic Mirror-Prox}%
\label{sec:stochmp-proof}

\TheoremStochMP*

Throughout this section, we will use the definitions of
$\Lambda_\gamma$ and $\mathcal{B}$ from~\cref{eq:dual-domain}
and~\cref{eq:primal-domain}, respectively.  We will also define the
composite space $Z \defeq \Lambda_\gamma \times \mathcal{B}$.  We will
use the norm $\norm{1}{\vlambda}$ for $\vlambda\in\Lambda_\gamma$,
whose dual norm is $\norm{\infty}{\wildcard}$. For
$\vtheta\in\mathcal{B}$ we will use the norm
$\norm{}{\vtheta}\equiv\norm{2}{\vPhi\vtheta}$ --- the corresponding
dual norm enjoys the convenient bound
$\norm{*}{\transp{\vPhiCore}\vec u} 
=\sup_{ \norm{}{\vtheta} \le 1} \transp{\vec u} \vPhiCore \vtheta 
\le \norm{2}{\vec u}
\le \norm{1}{\vec u}$ for any vector $\vec u$.%
\footnote{More generally,
  $\norm{*}{\vxi} = \inf\set{\norm{2}{\veta} \given
    \transp\vxi = \transp\veta \vPhiCore}$, which is
  non-zero when $\vxi\ne\vec0$ and $\vPhiCore$ has full column rank.} %
The last inequality is due a general property of $p$-norms:
$\norm{p}{\vec u} \le \norm{q}{\vec u}$ whenever $\infty\ge p \ge q\ge 1$.

\subsubsection{Lipschitz Constants}

Our first step will be to bound the Lipschitz constants associated
with $f$, the objective function of~\cref{eq:corelp}.  In other
words, we are looking for bounds on
$\norm{\infty}{f_{\vlambda}(\vtheta)}$ and $\norm{*}{f_{\vtheta}(\vlambda)}$.

First, for any $\vtheta\in\mathcal{B}$ we have
\begin{align*}
  \norm{\infty}{f_{\vlambda}(\vtheta)}
  &= \norm{\infty}{\vW\vecr + \vW(\gamma\vP - \vE)\vPhi\vtheta} \\
  &\le \norm{\infty}{\vW\vecr} + \norm{\infty}{\vW(\gamma\vP - \vE)\vPhi\vtheta}
  &\text{(by the triangle inequality)} \\
  &\le 1 + \max_i \norm{1}{\vW_i(\gamma\vP - \vE)}\norm{\infty}{\vPhi\vtheta}.
  &\text{(by definition of $\norm{\infty}{\wildcard}$ and H\"older's inequality)} \\
  \intertext{Now, by the property of norms,
  $\norm{\infty}{\vPhi\vtheta} \le \norm{2}{\vPhi\vtheta} \le B$.
  Secondly, $\vW_i\vP$ and $\vW_i\vE$ are probability distributions, so
  $\norm{1}{\vW_i(\gamma\vP - \vE)} \le \gamma\norm{1}{\vW_i\vP} +
  \norm{1}{\vW_i\vE} = 1+\gamma$ and}
  \norm{\infty}{f_{\vlambda}(\vtheta)}
  &\le 1 + (1 + \gamma)B
  \le 2B. &\text{(since $B \ge 1/(1-\gamma)$)}
\end{align*}

For the other gradient, we use the bound on dual norms mentioned
above:
\begin{align*}
  \norm{*}{f_{\vtheta}(\vlambda)}
  &= \norm{*}{(\transp\vphi_{s_0} + \transp\vlambda\vW(\gamma\vP - \vE))\vPhi} \\
  &\le \norm{1}{\transp\vece_{s_0}} + \norm{1}{\transp\vlambda\vW(\gamma\vP - \vE)} \\
  &\le 1 + \frac{1+\gamma}{1-\gamma} = \frac{2}{1-\gamma}\,,
\end{align*}
where the last inequality uses the fact that $(1-\gamma)\vlambda$ is a
probability distribution, as are the rows of $\vW$, $\vP$, and $\vE$.

\subsubsection{Gradient Estimator Variance}

Next, we will bound the variance in the stochastic estimators
$\hat{f}_{\vlambda}(\vtheta)$ and $\hat{f}_{\vtheta}(\vlambda)$ defined
in~\cref{eq:flambda-est} and~\cref{eq:ftheta-est}, respectively,
compared to the true gradients $f_{\vlambda}(\vtheta)$ and
$f_{\vtheta}(\vlambda)$ defined in~\cref{eq:flambda}
and~\cref{eq:ftheta}, respectively.

First, we bound
$\Ex{\norm{\infty}{\hat f_{\vlambda}(\vtheta) -
    f_{\vlambda}(\vtheta)}^2}$ for any $\vtheta\in\Bee$ by bounding
its components.  For any state $s\in\States_+$,
action $a\in\Actions$, and reward $\hat r$, we have
\begin{align*}
  \MoveEqLeft \abs{{[\hat f_{\vlambda}(\vtheta)]}_{sa} - {[f_{\vlambda}(\vtheta)]}_{sa}} \\
  &= \abs{(\hat r + \gamma\transp\vphi_{s'}\vtheta - \transp\vphi_s\vtheta)
    - (r_{sa} + \gamma\vP_{sa}\vPhi\vtheta - \transp\vphi_s\vtheta)}
  &\text{(for some random $s'\sim\vP_{sa}$)}\\
  &\le \abs{\hat r - r_{sa}} + \gamma\abs{\transp\vphi_{s'}\vtheta - \vP_{sa}\vPhi\vtheta}
  &\text{(by the triangle inequality)}\\
  &\le 2 + \gamma\abs{(\transp\vece_{s'} - \vP_{sa})\vPhi\vtheta} 
  &\text{(bounded rewards)} \\
  &\le 2 + \gamma\norm{1}{\transp\vece_{s'} - \vP_{sa}}\norm{\infty}{\vPhi\vtheta}
  &\text{(using H\"older's inequality)}\\
  &\le 2 + 2\gamma B
  &\text{(since $\norm{\infty}{\vPhi\vtheta} \le \norm{2}{\vPhi\vtheta} \le B$)}\\
  &\le 2B. &\text{(since $B \ge 1/(1-\gamma)$)}
\end{align*}
It follows that
$\norm{\infty}{\hat f_{\vlambda}(\vtheta) - f_{\vlambda}(\vtheta)}^2
\le {(2B)}^2$, and the same bound must hold for the expectation.

We will now bound the other gradient, using the following property of
Euclidean norms: for any vector-valued random variable $\vec h$ with
mean $\bar{\vec h}$,
$\Ex{\norm{}{\vec h - \bar{\vec h}}^2} = \Ex{\norm{}{\vec h}^2} -
\norm{}{\bar{\vec h}}^2 \le \Ex{\norm{}{\vec h}^2}$.  Then, for a
random choice of state $s\in\States_+$, action $a\in\Actions$, and next
state $s'\sim\vP_{sa}$:
\begin{align*}
  \Ex{\norm{*}{\hat f_{\vtheta}(\vlambda) - f_{\vtheta}(\vlambda)}^2}
  &= \Ex[\big]{\norm[\big]{*}{(\transp\vphi_0 + \norm{1}{\vlambda}(\gamma\vphi_{s'} - \vphi_s))
    - (\transp\vphi_0 + \transp\vlambda\vW(\gamma\vP - \vE)\vPhi)}^2} \\
  &= \paren*{\frac{1}{1-\gamma}}^2 \Ex[\big]{\norm[\big]{*}{(\gamma\vphi_{s'} - \vphi_s)
    - \transp{(\vlambda/\norm{1}{\vlambda})}\vW(\gamma\vP - \vE)\vPhi}^2} \\
  \intertext{Now, since
  $(s,a)\sim\vlambda/\norm{1}{\vlambda}$ and $s' \sim \vP_{sa}$, we
  have $\Ex{\gamma\vphi_{s'} - \vphi_s} =
  \transp{(\vlambda/\norm{1}{\vlambda})}\vW(\gamma\vP - \vE)\vPhi$, so
  by the above property of variance for vector-valued random
  variables,}
  &\le \paren*{\frac{1}{1-\gamma}}^2 \Ex{\norm{*}{\gamma\vphi_{s'} - \vphi_s}^2} \\
  &= \paren*{\frac{1}{1-\gamma}}^2 \Ex{\norm{*}{(\gamma\vece_{s'} - \vece_s)\vPhi}^2} \\
  &\le \paren*{\frac{1+\gamma}{1-\gamma}}^2 \le \paren*{\frac{2}{1-\gamma}}^2\,.
\end{align*}

\subsubsection{Distance-Generating Functions}

The Stochastic Mirror-Prox algorithm requires strongly convex
\emph{distance-generating functions} for $\Lambda_\gamma$ and $\Bee$
with respect to their respective norms.  A function
$\omega:\mathcal{X}\to\Real$ (with domain $\mathcal{X}\subset\Real^n$)
is said to be $\sigma$-\emph{strongly convex} (where $\sigma > 0$ is
called the modulus of convexity) with respect to a norm
$\norm{}{\wildcard}$ on $\mathcal{X}$ if any of the following
conditions hold for all $\vecx,\vecy\in\mathcal{X}$
\begin{enumerate}[(i), left=0pt]
\item\label{eq:sc-def} For all $\alpha\in[0,1]$, $\alpha\omega(\vecx) +
  (1-\alpha)\omega(\vecy) \ge \omega(\alpha\vecx + (1-\alpha)\vecy) +
  \sigma\alpha(1-\alpha)\norm{}{x-y}^2/2$.
\item\label{eq:sc-zero} $\omega$ is convex and $\omega(\vecx) \ge \omega(\vecy) +
  \innerp{\nabla\omega(\vecy)}{\vecx-\vecy} + \sigma\norm{}{\vecx -
    \vecy}^2/2$.
\item\label{eq:sc-first} $\mathcal{X}$ is convex and
  $\innerp{\nabla\omega(\vecx) - \nabla\omega(\vecy)}{\vecx-\vecy} \ge
  \sigma\norm{}{\vecx-\vecy}^2$.
\end{enumerate}
Condition~\ref{eq:sc-def} is the definition of strong convexity;
note that it reduces to convexity when $\sigma=0$.
Conditions~\ref{eq:sc-zero} and~\ref{eq:sc-first} are equivalent to
the definition under appropriate differentiability conditions on
$\omega$ that hold in our setting and when $\vecx,\vecy$ are in the
interior of $\mathcal{X}$; see \citet{YuNegentropy} for details.
\Citet{JuditskyStochMP2011} uses ``strongly convex'' to mean that a
function is 1-strongly convex according to
condition~\ref{eq:sc-first}.

Define the \emph{divergence function:}
\begin{align*}
  D_{\!\omega}(\vecx, \vecy)
  &\defeq \omega(\vecx) - \omega(\vecy) - \innerp{\nabla\omega(\vecy)}{\vecx-\vecy}.
\end{align*}
When $\omega$ is convex, one can see that $D_{\!\omega}$ is always
non-negative, and by condition~\ref{eq:sc-zero} the $\sigma$-strong
convexity of $\omega$ is equivalent to
$D_{\!\omega}(\vecx,\vecy) \ge \sigma\norm{}{\vecx-\vecy}^2/2$.  We will
use this equivalence to establish the strong convexity of our
distance-generating functions below.  An important operation related
to distance-generating functions is the \emph{proximal projection}
onto $\mathcal{X}$ with respect to $\omega$:
\begin{align*}
  \Pi_\omega(\vecx, \vxi)
  &\defeq \argmin_{\vecy\in\mathcal{X}} D_{\!\omega}(\vecy,\vecx) + \innerp{\vxi}{\vecy},
    \quad\vxi\in\Real^n.
\end{align*}
The \emph{center} of $\mathcal{X}$ with respect to $\omega$ is defined
as $\vecx_0 \defeq \argmin_{\vecx\in\mathcal{X}}\omega(\vecx)$ and the
\emph{diameter} of $\mathcal{X}$ is
$\Omega_{\mathcal{X}} \defeq \sup_{\vecy\in\mathcal{X}}
\sqrt{2D_{\!\omega}(\vecy,\vecx_0)}$.

A strongly convex distance-generating function can be thought of as a
generalization of the squared norm $\norm{}{\wildcard}^2$ --- the
corresponding divergence generalizes the squared distance function
$\norm{}{\vecx - \vecy}^2$; unlike the squared distance, however, the
divergence may not be symmetric.  Indeed, when $\norm{}{\wildcard}$ is
an Euclidean norm, and \emph{only} for such norms, the function
$\omega(\vecx) = \norm{}{\vecx}^2/2$ is 1-strongly convex
\citep[Proposition~2]{YuNegentropy}.  In this special case, the
divergence is
$D_{\!\omega}(\vecx, \vecy) = \norm{}{\vecx - \vecy}^2/2$ and the
proximal projection is simply the Euclidean projection:
$\Pi_\omega(\vecx,\vxi) = \argmin_{\vecy\in\mathcal{X}} \norm{}{\vecx
  + \vxi - \vecy}$.

Thus, since the domain $\Bee$ of the primal variables $\vtheta$ is
equipped with the Euclidean norm
$\norm{}{\vtheta} \equiv \norm{2}{\vPhi\vtheta}$, we will use the
1-strongly convex distance-generating function
$\omega_{\Bee}(\vtheta) = \norm{}{\vtheta}^2/2$.  Since $\Bee$ is the
Euclidean ball under this norm, the center of $\Bee$ is
$\vtheta_0=\vec0$ and its ``diameter'' (actually the radius, in this
case) is $\Omega_\Bee = B$; $\vtheta_0$ is used as the initial value
of $\vtheta$ in \cref{alg:corestomp}.  The proximal projection is
\begin{align*}
  \Pi_\Bee(\vtheta, \vxi)
  &= \argmin_{\vtheta'\in\Bee}\, \norm{}{\vtheta + \vxi - \vtheta'}
  = \frac{\vtheta + \vxi}{\max\set{1, \norm{}{\vtheta+\vxi}/B}}.
\end{align*}

For the dual variables $\vlambda\in\Lambda_\gamma$, our
distance-generating function is a modification of the
\emph{unnormalized negentropy}
$h(\vlambda) = \sum_i \lambda_i(\log \lambda_i - 1)$.  It is
well-known that this function is 1-strongly convex on the set
$\set{\vlambda \ge \vec0 \given \norm{1}{\vlambda} \le 1}$
\citep[e.g.,][Theorem~5]{YuNegentropy}.  To achieve 1-strong convexity
on $\Lambda_\gamma$ (where $\norm{1}{\vlambda} = 1/(1-\gamma) > 1$),
we use a modified form of this function:
\begin{align*}
h_\gamma(\vlambda) &\defeq \frac{h((1-\gamma)\vlambda)}{{(1-\gamma)}^2}.
\end{align*}
It follows that
$\nabla h_\gamma(\vlambda) = \brck{\nabla
  h((1-\gamma)\vlambda)}/(1-\gamma)$.  Thus, defining
$D_\Lambda(\vlambda',\vlambda) \defeq
D_{h_\gamma}(\vlambda',\vlambda)$, we have
\begin{align*}
  D_\Lambda(\vlambda', \vlambda)
  &= \paren*{\frac{1}{1-\gamma}}
    \brck*{ \frac{h((1-\gamma)\vlambda')}{1-\gamma} - \frac{h((1-\gamma)\vlambda)}{1-\gamma}
    - \innerp{\nabla h((1-\gamma)\vlambda)}{\vlambda' - \vlambda}} \\
  &= \paren*{\frac{1}{1-\gamma}}
    \brck*{ \frac{h((1-\gamma)\vlambda')}{1-\gamma} - \frac{h((1-\gamma)\vlambda)}{1-\gamma}
    - \frac{\innerp{\nabla h((1-\gamma)\vlambda)}{(1-\gamma)\vlambda' - (1-\gamma)\vlambda}}{1-\gamma}} \\
  \intertext{Now, since $(1-\gamma)\vlambda, (1-\gamma)\vlambda'
  \in\Delta_{\States_+\times\Actions}$ and $h$ is strongly convex on this set, we have}
  D_\Lambda(\vlambda', \vlambda)
  &\ge \paren*{\frac{1}{1-\gamma}}^2 \frac{\norm{1}{(1-\gamma)\vlambda - (1-\gamma)\vlambda'}^2}{2}
    = \frac{\norm{1}{\vlambda - \vlambda'}^2}{2}.
\end{align*}
Thus we have shown that $h_\gamma$ is 1-strongly convex on
$\Lambda_\gamma$.  By the properties of the negentropy function, we
can verify that $h_\gamma(\vlambda)$ is minimized for
$\vlambda_0 = \vec1_\Actions/A \oplus \gamma\vec1_{mA}/(1-\gamma)mA$,
i.e.\@ the uniform distribution over actions concatenated with the
scaled uniform distribution over state-action pairs in
$\CoreStates\times\Actions$ --- this value is used as the initializer in
\cref{alg:corestomp}.  Conversely, $h_\gamma(\vlambda)$ is maximized
for $\bar\vlambda = \vece_{a} \oplus \gamma\vece_{s'a'}/(1-\gamma)$,
i.e.\@ when $\bar\vlambda$ is concentrated on $(s_0,a)$ for some
$a\in\Actions$ and some $(s',a')\in\CoreStates\times\Actions$.  We can
verify through a short calculation that
\begin{align*}
  D_\Lambda(\bar\vlambda, \vlambda_0) &= \frac{\ell}{{(1-\gamma)}^2},
  &\text{where } \ell &\defeq \log A + \gamma\log m \\
  \Omega_\Lambda &= \sqrt{2D_\Lambda(\bar\vlambda, \vlambda_0)} = \frac{\sqrt{2\ell}}{1-\gamma}.
\intertext{Finally, the proximal projection onto $\Lambda_\gamma$ with respect to $h_\gamma$ is}
  \Pi_\Lambda(\vlambda,\vrho)
  &= \frac{\tilde\vlambda_0}{\norm{1}{\tilde\vlambda_0}}
    \oplus \frac{\gamma\tilde\vlambda_*}{(1-\gamma)\norm{1}{\tilde\vlambda_*}},
  &\text{where } \tilde\vlambda &\defeq \exp(\log\vlambda + \vrho),
\end{align*}
where $\tilde\vlambda_{s_0} \defeq {[\tilde\lambda_{s_0a}]}_{a\in\Actions}$
and $\tilde\vlambda_* \defeq {[\tilde\lambda_{sa}]}_{s\in\CoreStates,a\in\Actions}$,
so that $\tilde\vlambda = \tilde\vlambda_{s_0} \oplus \tilde\vlambda_*$.

\subsubsection{The Composite Space}

We will now gather together the preceding results and use them to
construct a norm and distance-generating function on the composite
optimization domain
$Z = \Lambda_\gamma \times \Bee \subset \Real^{(1+m)A} \oplus
\Real^d$.  We closely follow the construction of \citet[\S
4.2]{JuditskyStochMP2011}.  First, we define the squared norm:
\begin{alignat*}{2}
  \norm{}{\vlambda \oplus \vtheta}^2
  &\defeq{}\; \Omega_\Lambda^{-2} \norm{1}{\vlambda}^2 + \Omega_\Bee^{-2} \norm{}{\vtheta}^2
    &&=\; \paren*{\frac{1-\gamma}{\sqrt{2\ell}} \norm{1}{\vlambda}}^2 + \paren*{\frac{1}{B} \norm{}{\vtheta}}^2. \\
  \shortintertext{The corresponding squared dual norm is}
  \norm{*}{\vrho \oplus \vxi}^2 &\defeq{}\;
  \Omega_\Lambda^2 \norm{\infty}{\vlambda}^2 + \Omega_\Bee^2 \norm{*}{\vtheta}^2
    &&=\; \paren[\bigg]{\frac{\sqrt{2\ell}}{1-\gamma}\norm{\infty}{\vlambda}}^2 + \paren{B\norm{}{\vtheta}}^2.
\end{alignat*}
Define the operator $F(\vlambda, \vtheta) \defeq -f_{\vlambda}(\vtheta)
\oplus f_{\vtheta}(\vlambda)$.  Its Lipschitz constant with respect to this norm is
\begin{align*}
  \norm{*}{F(\vlambda,\vtheta)}
  &= \sqrt{ \frac{2\ell\norm{\infty}{f_{\vlambda}(\vtheta)}^2}{{(1-\gamma)}^2} + B^2\norm{*}{f_{\vtheta}(\vlambda)}^2} \\
  &\le \sqrt{ \frac{8\ell B^2}{{(1-\gamma)}^2} + \paren*{\frac{2B}{1-\gamma}}^2 } \\
  &= \frac{2B\sqrt{1 + 2\ell}}{1-\gamma} = C.
\end{align*}
Similarly, define the estimator
$\hat F(\vlambda, \vtheta) \defeq - \hat f_{\vlambda}(\vtheta) \oplus
\hat f_{\vtheta}(\vlambda)$.  Its variance enjoys
\begin{align*}
  \Ex{\norm{*}{\hat F(\vlambda,\vtheta) - F(\vlambda,\vtheta)}^2} &\le C^2.
\end{align*}
We do not repeat the calculation because it is identical to the
previous one, since the variances of our estimators have the same
bounds as their squared Lipschitz constants.

Finally, we construct the composite distance-generating function:
\begin{align*}
  \omega_Z(\vlambda, \vtheta) &\defeq \frac{{(1-\gamma)}^2 h_{\gamma}(\vlambda)}{2\ell} + \frac{\norm{}{\vtheta}^2}{2B^2}.
\end{align*}
We can verify that this function is 1-strongly convex on $Z$ with
respect to the norm defined above, and that the diameter of $Z$ under
this function is $\Omega \le \sqrt{2}$.

\begin{proof}[Proof of \cref{lemma:corestomp-subopt}]
  We apply the result of \citet[Corollary~1]{JuditskyStochMP2011} to
  our setting, where the Lipschitz constant of $F$ is $C$, the
  variance of $\hat F$ is $C^2$, and the diameter of $Z$ is
  $\Omega \le \sqrt{2}$.  Then the result tells us that a suitable
  learning rate is $\eta = \inv{C}\sqrt{2/7T}$, as specified in
  \cref{thm:corestomp}, and the resulting bound on the expected
  duality gap after $T$ iterations is
  \begin{align*}
    \epsopt &\defeq \Ex{\delta_{\Bee}(\hat\vlambda,\hat\vtheta)} \le \frac{14C}{\sqrt{3T}}. \qedhere
  \end{align*}
\end{proof}

\end{document}

%%% Local Variables:
%%% mode: latex
%%% TeX-master: t
%%% End: